%% file: main.tex
\documentclass{article}

\newcommand{\ignore}[1]{}
\usepackage{arxiv}
\usepackage{amsmath,amssymb,amsthm}
\usepackage[utf8]{inputenc} 
\usepackage[T1]{fontenc}    
\usepackage{url}            
\usepackage{booktabs}       
\usepackage{amsfonts}       
\usepackage{nicefrac}       
\usepackage{microtype}      
\usepackage{atbegshi}%
\newtheorem{lemma}{Lemma}

\newtheorem{theorem}{Theorem}
\usepackage[paper=portrait,pagesize]{typearea}
\usepackage{mathtools}
\usepackage{xspace}
\usepackage{epsfig}
\usepackage{algorithm}
\usepackage[noend]{algorithmic}
\usepackage{color, soul}
\usepackage{url}
\usepackage{natbib}
\usepackage{graphicx}
\usepackage{tikz}
\usepackage{wasysym}
\usepackage{longtable}
\usepackage{array}
\usepackage{multirow}
\usepackage{wrapfig}
\usepackage{float}
\usepackage{colortbl}
\usepackage{pdflscape}
\usepackage{tabu}
\usepackage{threeparttable}
\usepackage{threeparttablex}
\usepackage[normalem]{ulem}
\usepackage{makecell}
\usepackage{xcolor}
\usepackage{csquotes}
\usepackage{lscape}
\usepackage{hyperref}
\usepackage{pgfplots}
\usepackage{rotating}
\usepackage{pdflscape}
\usepackage{tabulary}
\usepackage{subcaption}
\usepackage[short]{optidef}
\usepackage{amssymb,amsmath}
\usetikzlibrary{matrix}
\usepgfplotslibrary{groupplots}
\pgfplotsset{compat=newest}

\graphicspath{ {./images/} }

\newcommand{\ie}{i.\,e.\xspace}

\makeatletter
\tikzset{
    every picture/.style={
        execute at begin picture={
            \let\ref\@refstar
        }
    }
}
\makeatother

\title{Entropy-Based Evolutionary Diversity Optimisation\\ for the Traveling Salesperson Problem}
\date{}

\author{
 Adel Nikfarjam \\
Optimisation and Logistics\\School of Computer Science\\The University of Adelaide\\
  \texttt{adel.nikfarjam@adelaide.edu.au} \\
   \And
 Jakob Bossek \\
Statistics and Optimization\\Dept. of Information Systems\\University of M\"unster\\
  \texttt{jakob.bossek@wi.uni-muenster.de} \\
  \And
 Aneta Neumann \\
Optimisation and Logistics\\School of Computer Science\\The University of Adelaide\\
  \texttt{aneta.neumann@adelaide.edu.au} \\
    \And
 Frank Neumann \\
Optimisation and Logistics\\School of Computer Science\\The University of Adelaide\\
  \texttt{frank.neumann@adelaide.edu.au} \\
}

\begin{document}
\maketitle
\begin{abstract}
Computing diverse sets of high-quality solutions has gained increasing attention among the evolutionary computation community in recent years. It allows practitioners to choose from a set of high-quality alternatives.
In this paper, we employ a population diversity measure, called the high-order entropy measure, in an evolutionary algorithm to compute a diverse set of high-quality solutions for the Traveling Salesperson Problem. In contrast to previous studies, our approach allows diversifying segments of tours containing several edges based on the entropy measure. We examine the resulting evolutionary diversity optimisation approach precisely in terms of the final set of solutions and theoretical properties. Experimental results show significant improvements compared to a recently proposed edge-based diversity optimisation approach when working with a large population of solutions or long segments.
\end{abstract}


\section{Introduction}
The classical optimisation task usually aims to find an (approximately) optimal solution regarding one or more objectives~\cite{DBLP:books/cu/BV2014,DBLP:books/tf/18/2018aam-1}. Evolutionary computation gained interest to compute diverse sets of high-quality solutions differing in terms of one or more structural features for a given optimisation problem~\cite{viet2020evolving, mouret2020quality, arulkumaran2019alphastar, bossens:hal-02555231, fontaine2019}. Generating a diverse set of high-quality solutions provides different implementation alternatives and enables further discussions on solution properties among stakeholders.
Multiple applications of using a diverse set of solutions can be found in the literature, such as robotics \cite{mouret2020quality} and video games \cite{bloem2014air, arulkumaran2019alphastar}.  

 
Several studies aim to find a diverse set of solutions using Evolutionary Algorithms~(EAs).
EAs~\cite{eiben2015introduction} provide us with useful solutions when facing a complex, weakly understood and/or a black-box problem; one can usually gain information on the objective function of such problems only by evaluation. EAs are population-based meta-heuristics, which adopt concepts and mechanisms inspired by natural evolution such as mutation, crossover and (survival) selection. 
Diversity preservation mechanisms are generally incorporated into EAs to prevent premature convergence~\cite{neumann2019evolutionary} and often enable the algorithms to (easily) escape local optima. Recently, diversity was adopted in the context of \emph{Evolutionary Diversity Optimisation} (EDO) for a different reason. Here, the goal is to find a set of solutions with desirable objective values diversified with respect to (structural) properties of the solutions~\cite{neumann2018discrepancy,DBLP:journals/corr/abs-2010-11486}. The field was established by \citet{ulrich2011maximizing} who first proposed an EA to evolve diverse sets of high-quality solutions in the continuous domain.
Recent studies in EDO focused on evolving sets of benchmark instances for the Traveling Salesperson Problem~(TSP) with diverse characteristics important to understand the performance of TSP solvers as well as generating diverse (with respect to different aesthetics) sets of images~\cite{alexander2017evolution,doi:10.1162/evcoa00274}.

The incorporation of star-discrepancy and indicators from evolutionary multi-objective optimisation into evolving diverse sets of benchmarks for TSP and sets of images have been studied and assessed by~\citet{neumann2018discrepancy,neumann2019evolutionary}. A different approach to achieve high diversity in the feature-space of the TSP instance was proposed in~\cite{bossek2019evolving}; here, high diversity was achieved implicitly by more sophisticated mutation operators without explicit diversity-preserving mechanism within the introduced EA.

Several algorithms have been introduced to find the (approximately) optimal solution for TSP \cite{lin1973effective,helsgaun2000effective,xie2008multiagent,nagata2013powerful}. More recently, \citet{viet2020evolving} studied the problem of generating diverse sets of high-quality TSP solutions. The authors introduced two different diversity measures, edge diversity (ED) and pairwise distance (PD), based on pairwise edge overlap. They embedded $k$-OPT mutation operators with different values of $k$ in an EA introduced to solve the EDO problem, and empirically studied the mutations' impact on the EA performance. 

The diversity measures used by \citet{viet2020evolving} do not consider dependencies between the decision variables of TSP. This is while the value of a decision variable in TSP is strongly correlated to the value of other variables \cite{nagata2020high} (we will explain this matter further in the Section~\ref{Sec:problem}). Therefore, we incorporate a diversity measure based on entropy into EDO for the TSP in this study. This measure explicitly addresses the dependency between decision variables. We examine the diversity measure's theoretical properties and determine characteristics that a maximally/minimally diverse set of tours should possess.

Besides, we propose a Mixed-Integer Programming~(MIP) formulation of the considered diversity problem and solve it with an exact solver to a) support the theoretical proofs and b) use it as a baseline for experimentation. Then, we introduce the biased 2-OPT mutation, which mainly focuses on more frequent components in the population, and aims to decrease their frequency to increase diversity. Finally, we perform an extensive experimental study in the unconstrained case (no quality criterion) and the constrained case with (un)biased 2-OPT mutation operators. 
Our results indicate a clear advantage of the entropy-based driven EA compared to EAs based on the distance-based diversity measures introduced by~\citet{viet2020evolving}. The results also show that using biased 2-OPT brings about faster convergence, especially in unconstrained diversity optimisation.

The remainder of this paper is structured as follows. In Section~\ref{Sec:problem}, we describe the problem and three different diversity measures for TSP tours. Next, we provide the theoretical properties of the entropy-based diversity measure in Section~\ref{sec:Entopy}. A MIP formulation and an EA are introduced in Section~\ref{sec:MIP} and~\ref{sec:alg}, respectively. Afterwards, we conduct a series of experiments for unconstrained and constrained diversity optimisation to compare the performance of the high-order entropy measure, the EA, and biased 2-OPT to previously used measures and algorithms. Finally, we finish with some concluding remarks and ideas for future research.

\section{Maximising Diversity in TSP}
\label{Sec:problem}

The TSP is a well-known NP-hard combinatorial optimisation problem. The problem is defined on a directed complete graph $G=(V,E)$ where $V$ is a set of nodes and $E$ is a set of pairwise edges between the nodes, $e = (i, j) \in E$, each associated with a positive weight, $d(e)$. In this paper, we assume that the TSP instances are symmetric (\ie $d(i, j) = d(j, i)$). We denote by $n=|V|$ and $m=|E|=n(n-1)/2$ the cardinality of these sets.
The objective is to compute the permutation $p : V \to V$ minimising the cost function:

$$c(p) = d(p(n),p(1)) + \sum_{i=1}^{n-1} d(p(i),p(i+1)).$$

In this study, we examine TSP in the context of EDO. Given a TSP instance $G$, let $OPT$ be the cost of the optimal tour for $G$ and $\alpha > 0$ be a predefined parameter. The objective is to compute a diverse set of tours where a) the diversity value of the population is maximised in terms of a given diversity measure; b) all individuals comply with a maximum cost (\ie $c(p_i)\leq OPT(1+\alpha), \forall p_i \in P$).  In other words, the goal is to maximise the diversity of the set of solutions subject to the quality constraint. Maximising the diversity of tours provides us with valuable information on solution space around the optimal tour. It can indicate which edges are irreplaceable or complex to replace if we want to stay within the quality threshold. Moreover, it enables decision-makers to choose between different tours; they may decide to visit a city earlier than another or avoid an edge if provided with various alternatives with reasonable costs.   

Recently, \citet{viet2020evolving} studied EDO on TSP for the first time. The authors tailored two edge-based diversity measures, ED and PD towards TSP.
ED measures the diversity based on the equalisation of the frequency of edges in the population. For this purpose, they used the notion of genotypic diversity~\cite{zhu2004population} defined as the mean of pairwise distances:
\begin{equation*}
ED(P) = \sum_{p \in P}\sum_{q \in P} |E(p)\setminus E(q)|,
\end{equation*}
where $E(p)$ is the set of edges of $p = (p(1), \ldots, p(n))$ (\ie $E(p)=\{(p(1), p(2)), (p(2), p(1)) \ldots, (p(n),p(1)), (p(1), p(n))\}$).

On the other hand, PD is defined as
\begin{equation*}
    PD(P) = \frac{1}{n\mu} \sum_{p \in P} min_{q \in P\setminus{p}} \{|E(p)\setminus E(q)|\}.
\end{equation*}
and emphasises uniform pairwise edge distances. PD is closely aligned with the diversity measure in~\citet{7473938}.

For the sake of brevity, we refer the reader to Do et al.~\cite{viet2020evolving} for further details.

One disadvantage of ED and PD is that the dependency of the occurrence of nodes in a tour (decision variables) is not considered. This is while the occurrence of nodes in a tour is significantly dependent on each other in TSP. Here, we show a tour as a permutation $p$ consisting $n$ decision variables $p(i)$ representing the $i$-th node visited in the tour. For instance, if we construct a tour manually, the next node we choose (the value of $p(i+1)$) is heavily dependent on the current node ($p(i)$) and all already visited nodes~\cite{nagata2020high}. This is because we cannot choose a visited node. This issue can result in an inaccurate diversity evaluation. We employ an entropy-based diversity measure introduced by Nagata~\cite{nagata2020high}, termed High-order entropy, to resolve this issue. The measure considers the sequence of $k$ nodes ($k-1$ edges) in tours instead of focusing on edges one by one. \citet{nagata2020high} showed that the High-order entropy measure outperforms the independent entropy measure in terms of preventing premature convergence.

\section{High-Order Entropy Measure} \label{sec:Entopy}
For the high-order entropy measure, the sequence of $2 \leq k \leq n$ nodes ($k-1$ edges) in tours is the feature intended to be diversified.
Let $s =\{v_1, \ldots, v_k\}, v_i \in V$ be a segment consisting of $k$ nodes. Then, its contribution to the overall entropy of the population $P$ is given as
$$h(s)= - \left(f(s)/ (2n\mu)\right)\ln{\left(f(s)/ (2n\mu)\right)},$$
where $f(s)$ is the absolute number of occurrences of segment $s$ in $P$. Note that $2n\mu$ is the total number of occurrences of all segments in a population of size $\mu$ when we are able to traverse each tour in both directions. Each tour contains exactly $2n$ different segments (see Figure~\ref{fig:sampling} for an example). In the following, it is sometimes useful to show a segment by means of its set of edges. For instance, $s = \{s(1), s(2), s(3)\}$ can be also shown as $E(s) = \{(s(1),s(2)), (s(2),s(3))\}$.

Summing over all segments included in the population $P$, the entropy of $P$ is defined as 
$$H(P) =  \sum_{s \in P} h(s).$$


Let $S = \{s_1, \ldots ,s_u\}$ be the set of all possible segments of $k$ nodes for a given TSP instance $G$, and 
$u=\frac{n!}{(n-k)!} = |S|$ denotes the cardinality of $S$. We sort the segments according to the number of their occurrences within $P$ in an increasing order to obtain the vector
$$F(P) = (f(s_1), \ldots, f(s_u)).$$ 
It means that $f(s_1) \leq f(s_2) \leq \ldots \leq f(s_u)$.
We define $f_{\min} = f(s_1)$, $f_{\max} = f(s_u)$, and $C=f_{\max}-f_{\min}$ where $f_{\min}$ and $f_{\max}$ are the smallest and the largest number of occurrences of segments in $P$, respectively.
Intuitively, a maximally diverse population would have all $f(s_i) \in F(P)$ 
almost equalised. We will use $F(P)$ later to analyse whether a given $P$ has the maximum achievable entropy.

\begin{figure}[t]
\centering
\input{pic1.tex}
\caption{Illustration of building all segments of length $k=3$ of a TSP tour.}
\label{fig:sampling}
\end{figure}
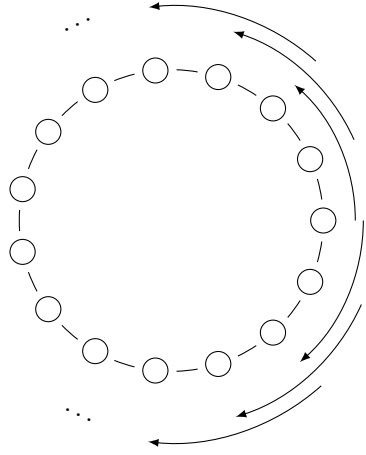
\subsection{Maximum High-Order Entropy} \label{subsec:max_H}
Next, we aim to determine the characteristics of an ideal set of tours having the maximum high-order entropy value $H_{\max}$ for a given TSP instance.
Knowing $H_{\max}$ is important for two main reasons: a) it enables us to have a better understanding of an algorithm's performance by comparing the entropy of the final population with $H_{\max}$ and b) it allows us to use it as a termination criterion for an EA in the course of experimental evaluation with a fixed-target perspective.
\begin{lemma}
\label{lem:equalizing_freq_better_entropy}
Let $P_2$ be a population obtained from a population $P_1$ by decreasing $f_{\max}$ and increasing $f_{\min}$ by one unit each. If $C\geq 2$, then we have $H(P_2) > H(P_1)$.   
\end{lemma}
In order to show Lemma~\ref{lem:equalizing_freq_better_entropy}, we work under the assumption that $C \geq 2$ and show that $H(P_2) - H(P_1)>0$ holds. We use that $H(P_2) - H(P_1)$ is monotonically decreasing in $f_{\max}$ and $\lim_{f_{\max} \rightarrow +\infty} H(P_2) - H(P_1)$ converges to zero. This implies that Lemma~\ref{lem:equalizing_freq_better_entropy} is true.
The differences in $P_1$ and $P_2$ can be summarised in $f_{\min}$ and $f_{\max}$ where $f_{\max}$ decreased and $f_{\min}$ increased by one unit each in $P_2$. The number of occurrences of other segments are the same in both populations. Also, we have $f_{\max} = f_{\min}+C$.
To simplify the following presentation, we use $f=f_{\max}$ and $f_{\min}=f-C$.
Thus, we have:\\
\begin{align*}
H(P_2) - H(P_1) = -\frac{f-1}{2n\mu}\ln{\left(\frac{f-1}{2n\mu}\right)}-\frac{f-C+1}{2n\mu}\ln{\left(\frac{f-C+1}{2n\mu}\right)}+\frac{f}{2n\mu}\ln{\left(\frac{f}{2n\mu}\right)} + \frac{f-C}{2n\mu}\ln{\left(\frac{f-C}{2n\mu}\right)} 
\end{align*}
We now show that $H(P_2) - H(P_1)$ is monotonically decreasing in $f$ if and only if $C \geq 2$.
\begin{lemma}
\label{lem:g(f)_descending}
If $C \geq 2$ then $H(P_2) - H(P_1)$ is monotonically decreasing in $f$.
\end{lemma}
\begin{proof}
To prove $H(P_2) - H(P_1)$ is monotonically decreasing, we show that $\frac{d(H(P_2) - H(P_1))}{df} < 0$. We have
\begingroup
\allowdisplaybreaks
\begin{align*}
&\frac{d(H(P_2) - H(P_1))}{df} < 0\\
\Leftrightarrow \quad &\frac{1}{2n\mu}\left(\ln \left(\frac{f}{2n\mu}\right)+\ln{\left(\frac{f-C}{2n\mu}\right)}\right)-\frac{1}{2n\mu}\left(\ln{\left(\frac{f-1}{2n\mu}\right)}+\ln{\left(\frac{f-C+1}{2n\mu}\right)}\right)<0 \\
\Leftrightarrow \quad &\ln{\frac{f(f-C)}{(f-1)(f-C+1)}}<0 \\
\Leftrightarrow \quad &\frac{f(f-C)}{(f-1)(f-C+1)}<1 \\
\Leftrightarrow \quad &(f-1)(f-C+1)>f(f-C)\\
\Leftrightarrow \quad &f^2-f\cdot C +C-1 > f^2-f\cdot C
\end{align*}
\endgroup
The last expression holds as $C \geq 2$, which completes the proof.
\end{proof}

Owing to Lemma~\ref{lem:g(f)_descending}, if $H(P_2) - H(P_1)$ is still positive for an extremely large $f$, it is positive for all smaller values of $f$. We now investigate $f$ approaching to $+\infty$.   

\begin{lemma}
\label{lem:largeF}
$H(P_2) - H(P_1) > 0$ holds for any fixed population size $\mu$ and $C \geq 2$. 
\end{lemma}
\begin{proof}
As one can notice, $f$ is bounded by $\mu$. that means $f$ can approaches to $+\infty$, only if $\mu$ approaches to $+\infty$ as well. Thus, we investigate $H(P_2) - H(P_1)$ in the most extreme case where $\mu$ and $f \rightarrow +\infty$.   
\begingroup
\allowdisplaybreaks
\begin{align*}
     \quad & H(P_2) - H(P_1)\\
    \Leftrightarrow \quad &\left(\frac{f-C}{2n\mu}\right)\ln{\left(\frac{f-C}{2n\mu}\right)} - \left(\frac{f-1}{2n\mu}\right)\ln{\left(\frac{f-1}{2n\mu}\right)} + \left(\frac{f}{2n\mu}\right)\ln{\left(\frac{f}{2n\mu}\right)}-\left(\frac{f-C+1}{2n\mu}\right)\ln{\left(\frac{f-C+1}{2n\mu}\right)}
\end{align*}
We compute the limit for $\mu$ and $f \rightarrow +\infty$ by applying L'Hopital's rule and have:
\begin{align*}
   \quad &\left(\frac{1}{2n}\right)\ln{\left(\frac{1}{2n}\right)} - \left(\frac{1}{2n}\right)\ln{\left(\frac{1}{2n}\right)} + \left(\frac{1}{2n}\right)\ln{\left(\frac{1}{2n}\right)}-\left(\frac{1}{2n}\right)\ln{\left(\frac{1}{2n}\right)} = 0
\end{align*}
\endgroup
The last expression shows that $H(P_2) - H(P_1)$ converges to $0$
if $f \rightarrow +\infty$. We have $f \leq \mu$ and
using Lemma~\ref{lem:g(f)_descending}, this implies that $H(P_2) - H(P_1)>0$ for any fixed $\mu$ 
if $C \geq 2$. 
\end{proof}
\ignore{
By adding $\ln{\left(\frac{f}{2n\mu}\right)}- \ln{\left(\frac{f}{2n\mu}\right)} + \ln{\left(\frac{f+C-1}{2n\mu}\right)} - \ln{\left(\frac{f+C-1}{2n\mu}\right)}=0$ to the left hand side of the previous equation, we have:
\begin{align*}
     \nonumber\quad &(f+1)\left(\ln{\left(\frac{f}{2n\mu}\right)}-\ln{\left(\frac{f+1}{2n\mu}\right)}\right)\\ \nonumber
    &-(f+C)\left(\ln{\left(\frac{f+C-1}{2n\mu}\right)}-\ln{\left(\frac{f+C}{2n\mu}\right)}\right)\\ \nonumber
    &-\ln{\left(\frac{f}{2n\mu}\right)}+ \ln{\left(\frac{f+C-1}{2n\mu}\right)} \\ \nonumber
    \Leftrightarrow \quad & \ln{\left(\frac{f}{f+1}\right)^{(f+1)}}+\ln{\left(\frac{f+C-1}{f}\right)}\\
    &-\ln{\left(\frac{f+C-1}{f+C}\right)^{(f+C)}} 
\end{align*}
Using $\lim_{x\rightarrow\infty}\left(1-\frac{1}{x}\right)^{x} = e^{-1}$, we have that $\lim_{f\rightarrow\infty}\left(\frac{f}{f+1}\right)^{(f+1)}= e^{-1}$ and also $\lim_{f\rightarrow\infty}\left(\frac{f+c-1}{f+c}\right)^{(f+c)}= e^{-1}$. Moreover, $\lim_{f\rightarrow\infty} \left(\frac{f+C-1}{f}\right) = 1$. Therefore, if $f$ approaches to $+\infty$, we have:
\begin{align*}
   \quad \ln{e^{-1}}+ \ln{1} - \ln{e^{-1}} = 0
\end{align*}

\endgroup
The last expression shows that $H(P_2) - H(P_1)$ converges to $0$ when f approaches to +$\infty$. By considering Lemma~\ref{lem:g(f)_descending} ($H(P_2) - H(P_1)$ is monotonically decreasing in $f$), we can argue that $H(P_2) - H(P_1)$ converges to $0$ from higher values if $C \geq 2$. In other words, $H(P_2) - H(P_1) > 0$ holds for all real values of $f$ if $C \geq 2$. 

}

\begin{theorem}
\label{thm:frequency_of_all_segments_for_H_max}
For every complete graph with $n$ nodes and every population size $\mu \geq 2$, the entropy of a population $P$ with $\mu$ individuals is maximum if and only if $C$ is equal to zero or one.
\end{theorem}
\begin{proof}
Lemma \ref{lem:equalizing_freq_better_entropy} shows that a population's entropy can be increased as long as $C \geq 2$. Therefore, $C$ should be set to 0 or 1 to have a maximum entropy population.    
\end{proof}
To set $C$ to 0 or 1, the number of occurrences of all possible segments should be equalised. For every TSP instance, there are $u$ possible segments, and $2n\mu$ occurrences of all segments for every population. The optimal value of $f_{\min}$ is equal to $[\frac{2n\mu}{u}]$. Let $f_{\min}^*$ and $C^*$ be the values of $f_{\min}$ and $C$ in an optimal population. It should be noted that based on the Pigeonhole principle, if $\frac{2n\mu}{u}$ is integer, $C \in \{0,2,3,\dots,u\}$ and $C^* = 0$; otherwise, $C \in \{1,2,\dots,u\}$ and $C^* = 1$. In other words, $C$ can get only one of the values of 0 or 1 depending on parameters of the problem such as the sizes of population, segments, and TSP instances. All in all, $(f_{\min}^*+1)u-(2n\mu)$ segments occur $f_{\min}^*$ times in an optimal population whereby, the number of occurrences of the other segments is equal to $f_{\min}^*+C^* = f_{\max}^*$.   
\begin{align}
    &H_{\max} = -((2n\mu)-(f_{\min}^*\cdot u))\left(\frac{f_{\max}^*}{2n\mu}\right)\ln{\left(\frac{f_{\max}^*}{2n\mu}\right)} - ((f_{\min}^*+1)u-(2n\mu))\left(\frac{f_{\min}^*}{2n\mu}\right)\ln{\left(\frac{f_{\min}^*}{2n\mu}\right)} \label{eq:max}
\end{align}

Note that the entropy of any set of TSP tours is always greater than zero. This is because no segments are allowed to occur within a tour more than once. In the worst-case scenario where a population consists of $\mu$ copies of a single tour, we have $2n$ different segments with the number of occurrences $\mu$. We can determine the entropy value of a population with such characteristics from:

 \begin{align}
     &H_{\min} = - 2n\left(\frac{1}{2n}\ln{\left(\frac{1}{2n}\right)}\right) = \ln(2n) \label{eq:min}
 \end{align}

\section{Mixed-Integer Programming Formulation}
\label{sec:MIP}

In this section, we give a MIP formulation for the considered problem. Solving the proposed MIP with an exact solver such as the Cplex solver \cite{cplex2009v12} can support the maximum entropy's proof. Also, it would provide us with a baseline for investigate the performance of other algorithms. The objective function is formulated as follows:
\begin{align}
H(P) = \sum_{s \in P} - \left(f(s)/ (2n\mu)\right)\ln{\left(f(s)/ (2n\mu)\right)} \to \max! \label{MIP:obj1}
\end{align}
where $f(s), s=\{v_i,\ldots, v_q\}$ is calculated from
\begin{align}
&\displaystyle f(s) = \sum_{p \in P} {x_{ij}^p \cdots x_{tq}^p} + \sum_{p \in P} {x_{ji}^p \cdots x_{qt}^p} \label{MIP:Con6}
\end{align}

Here, $x_{ij}^p$ is a binary variable; it is set to $1$ if edge $e = (i,j)$ is included in tour $p$; otherwise, it is equal to zero. For example, if $s = (v_3, v_5, v_2, v_1)$,
$f(s) = \sum_{p \in P} (x_{35}^p\cdot x_{52}^p\cdot x_{21}^p) + \sum_{p \in P} (x_{12}^p \cdot x_{25}^p \cdot x_{53}^p)$. The maximisation of the objective function in Eq.~\ref{MIP:obj1} is subject to the following constraints: 

\begin{align}
    &\displaystyle\sum_{i=1}^n\sum_{j=1}^n d(i,j)x_{ij}^p \leqslant (1+ \alpha)\cdot OPT,  ~~\forall p \in P \label{MIP:Con1}\\
    &\displaystyle\sum_{i=1,i\ne j}^n x_{ij}^p = 1, ~~\forall j \in V, p \in P \label{MIP:Con2}\\
    &\displaystyle\sum_{j=1,i\ne j}^n x_{ij}^p = 1, ~~\forall i \in V, p \in P \label{MIP:Con3}\\
    &\displaystyle w_i^p-w_j^p+nx_{ij}^p \leqslant n-1, \forall i, j \in V, i \ne j, p \in P \label{MIP:Con4}\\
    &\displaystyle w_i^p \leqslant n-1, ~~\forall i \in \{2,\ldots,n\}, p \in P \label{MIP:Con5}\\
    &\displaystyle x_{ij}^p \in \{0,1\}, w_i \geq 0, \forall i, j \in V, p \in P \label{MIP:ConV}.
\end{align}
Here, $w_i^p$ is a positive integer showing the position of node $i$ in the tour $p$. Equation~\ref{MIP:Con1} makes sure that all solutions satisfy a minimal quality with respect to tour length. Equations~\ref{MIP:Con2} and \ref{MIP:Con3} guarantee that all nodes are visited exactly once in each tour, while Equations~\ref{MIP:Con4} and \ref{MIP:Con5} prevent the creation of sub-tours as proposed by~\citet{miller1960integer}.
The objective function (Eq.~\ref{MIP:obj1}) should be linearised to use the MIP solvers.

\subsection{Linearisation} 

In Section~\ref{subsec:max_H}, we showed that the entropy value of a  population $P$ is maximum if and only if $C=f_{\max}-f_{\min} \in \{0,1\}$. In other words, Equation~\ref{MIP:obj1} is maximised if and only if $C$ is set to zero or one. Therefore, We can replace the objective function with
\begin{align}
   C = f_{\max}-f_{\min} \to \min! \label{MIP:obj2}
\end{align}
Here, $f_{\min}$ and $f_{\max}$ are dependent on $x_{ij}^p$. Thus, the correlation between these two variables and the other MIP variables should be explicitly defined in the MIP formulation before using Equation~\ref{MIP:obj2} as the MIP's objective function. For this purpose, we need to add new constraints and variables.
\begin{align}
    &\displaystyle y_{ij\cdots s}^p \geqslant x_{ij}^p+ \cdots +x_{tq}^p-k+2, ~~ \forall i,\ldots,q \in V, p \in P \label{MIP:Con9}\\
    &\displaystyle y_{ij\cdots s}^p \geqslant x_{qt}^p+ \cdots +x_{ji}^p-k+2, ~~ \forall i,\ldots,q \in V, p \in P \label{MIP:Con10}\\
    &\displaystyle \sum_{i}\sum_{j}\cdots\sum_{t}\sum_{p}y_{ij\cdots sq}^p \leqslant 2n\mu \label{MIP:Con11}\\
        &\displaystyle f_{\max} \geqslant \sum_{p \in P}y_{ij\cdots tq}^p, ~~ \forall i,\ldots,q \in V \label{MIP:Con7}\\
    &\displaystyle f_{\min} \leqslant \sum_{p \in P}y_{ij\cdots tq}^p, ~~ \forall i,\ldots,q \in V \label{MIP:Con8}
\end{align}
Here, $y_{ij\cdots tq}^p$ is a binary variable set to 1 if segment $s=\{v_i,\vdots,v_q\}$ or $s'=\{v_q,\vdots v_i\}$ is included in tour $p$. For example, if the tour $p$ includes either of the segments $s = \{v_3, v_5, v_2, v_1\}$ (\ie $E(s)=\{(3, 5), (5,2), (2,1)\}$) or $s' = \{v_1, v_2, v_5, v_3\}$, both $y_{3521}^p$ and $y_{1253}^p$ are set to 1. Note that segments $s$ and $s'$ are identical since we can traverse a tour in both directions. Equations~\ref{MIP:Con9} and \ref{MIP:Con10} ensure that $y_{ij\cdots q}^p$ is set to $1$ if segment $s$ is included in tour $p$. Equation~\ref{MIP:Con11} guarantees that $y_{ij\cdots q}^p$ is equal to zero if the segment $s$ is not included in the tour $p$. Finally, Equations~\ref{MIP:Con7} and \ref{MIP:Con8} determine $f_{\min}$ and $f_{\max}$.
Moreover, $f(s)$ can be calculated from summing up $y_{ij\cdots q}^p$ over $p$. In the final MIP formulation, Equation~\ref{MIP:obj2} serves as the objective function subject to the constraints [\ref{MIP:Con1}-\ref{MIP:ConV}] and [\ref{MIP:Con9}-\ref{MIP:Con8}]. 

\section{Entropy-based Evolutionary Diversity Optimisation}
\label{sec:alg}

We introduce an EA to address EDO for TSP tours (see Algorithm~\ref{alg:ea} for an outline). The algorithm is initialised with a population $P$ consisting of $\mu$ copies of an optimal tour/permutation for the given TSP instance. A broad range of successful algorithms is proposed in the literature to find the optimal tour in TSP, such as Concorde by \citet{applegate2003implementing}. Moreover, the optimal tours have been provided for most benchmark instances in the well-known TSPlib~\cite{Reinelt91tsplib}. Next, a parent $p$ is selected uniformly at random, and mutation operators generate two offspring individuals, $p'$ and $p''$, one by biased 2-OPT and the other by classic 2-OPT. The offspring by biased 2-OPT is more likely to contribute to the population's entropy, while the other stands a higher chance to comply with the quality criterion. Having used both operators simultaneously, we increase the likelihood of a successful iteration. Having removed $p$ from $P$, add it to $P'$ (the survival selection's pool). If the length of $p'$ is larger than $(1+\alpha)\cdot OPT$, $p'$ is discarded; otherwise, it is added to $P'$. This step is repeated for $p''$. Afterwards, an individual $p^*$ is selected from $P'$ where $H(P \cup \{p^*\})$ is maximum; add $p^*$ to $P$. From the entire population, we solely consider parents for survival selection to increase time efficiency. We will discuss that the exclusion of the rest of the population does not affect the results significantly.
These steps are repeated until a termination criterion is met.

\begin{algorithm}[t]
\begin{algorithmic}[1]
\STATE Initialise the population $P$ with $\mu$ TSP tours such that $c(p) \leq   (1+ \alpha)\cdot OPT$ for all $p \in P$.\\
\STATE Choose $p \in P$ uniformly at random and produce two offspring $p'$ and $p''$ of $p$ by biased 2-OPT and classic 2-OPT.\\ 
\STATE Remove $p$ from $P$ and set $P' =\{p\}$.\\
\STATE If $c(p') \leq (1+ \alpha)\cdot OPT$, add $p'$ to $P'$. \\
\STATE If $c(p'') \leq (1+ \alpha)\cdot OPT$, add $p''$ to $P'$.\\
\STATE Select an individual $p^*$ from $P'$ where $p^* = \text{arg$\,$max}_{p^* \in P'} \{H(P \cup \{p^*\})\}$ and add $p^*$ to $P$.\\
\STATE Repeat steps 2 to 6 until a termination criterion is reached.
\end{algorithmic}
\caption{Diversity maximising EA}
\label{alg:ea}
\end{algorithm}
\subsection{Biased 2-OPT}
We introduce two biased versions of $2$-OPT mutation. In the classic $2$-OPT, two nodes are selected randomly. These two nodes are swapped, and all nodes between them are sorted in the backward direction. Since our focus is on the symmetric TSP, the difference between the parent and offspring is solely in the two edges where the swap takes place. The biased versions introduce a bias into the classic $2$-OPT. Here, the population's high frequent segments are more likely to be selected as sources for swaps. In the \emph{normalised biased $2$-opt} version, there is a competition based on the number of occurrences of segments where a segment's likelihood is proportional to its frequency. The \emph{absolute biased $2$-OPT} only selects the segment with the highest occurrences. The absolute biased $2$-opt is used in unconstrained diversity optimisation where the diversity is not subject to the quality constraint. This is while the normalised biased version is utilised in constrained diversity optimisation since focusing only on the most frequent segments decreases the probability of generating an offspring compatible with the quality criterion. 

Owing to the parent and the offspring's similarity, the algorithm compares the offspring with its parent rather than the entire population. All versions of $2$-OPT change solely two edges of a parent. Thus, the population's entropy is likely to decrease if both parent and offspring remain in the population, especially in unconstrained diversity optimisation. In constrained diversity optimisation, it can sometimes improve the results slightly if we compare an offspring to the entire population. However, it significantly increases the computational costs. This is because the latter survival selection requires updating every individual's contribution to the entropy whenever an offspring is generated. More specifically, the EA needs $\mu$ more diversity re-evaluation per generated offspring if it compares an offspring with the entire population. Since the re-evaluation is computationally expensive, it can affect the algorithm's time efficiency, especially when $\mu$ is large. The re-evaluation can be avoided by comparing the offspring to the parent solely.
\section{Experimental Investigation} \label{sec:experiments}
We conduct a series of experiments to evaluate the suitability of the proposed algorithm and diversity measure. The experiments are classified into three parts. First, we examine the algorithm's results to make sure that a) the considered survival selection does not affect the entropy of the final population by comparing the results with an EA including the entire population in the survival selection procedure and b) the results obtained from our EA are consistent with the results of the Cplex solver 
and $H_{\max}$~(see Eq.~\ref{eq:max}). Subsections~\ref{subsec:Uncon_Res} and \ref{subsec:con_Res} are dedicated to comparing the introduced EA and the EAs based on PD and ED by~\citet{viet2020evolving} in unconstrained diversity optimisation and constrained diversity optimisation, respectively.
\subsection{Validation of the Proposed EA}
\subsubsection{Survival Selection Procedure}
As mentioned, it is more efficient to compare the offspring with the parent than the entire population, especially in unconstrained optimisation.
Here, we analyse the algorithm's survival selection and compare it with the same algorithm where the offspring is compared with all individuals in the population. All combinations of $n = \{25, 50\}$, $\mu = \{12, 20, 50\}$, and $k = \{2, 3, 4\}$ are subject to experimentation. Due to the relaxation of the quality constraint, we consider complete graphs where the edges' weight are all equal to one, as TSP instances for unconstrained diversity optimisation.
The termination criterion is reaching the limit of $100\,000$ generated offspring. The results show no significant differences in the mean of entropy values over all the cases. The observation is confirmed with the Kruskal-Wallis test at significance level $95\%$ and the Bonferroni correction method. However, the first selection procedure avoids $\mu$ entropy re-evaluations per cost evaluation. This makes the EA considerably more efficient. For instance, the mean of CPU time is $45$ seconds for the first selection procedure where $n = 50$, $\mu = 50$, and $k = 2$, while the figure stands at $267$ seconds for the second selection procedure. 
\subsubsection{Comparison between the exact solver and the proposed EA}
We consider unconstrained diversity optimisation to investigate the results obtained from solving the MIP formulation by the Cplex solver. The constrained optimisation is not taken into account in this section for two main reasons. The main reason for using an exact solver such as Cplex is to support the formula provided for $H_{\max}$ such that we can use the $H_{\max}$ as the baseline for larger instances where the Cplex solver is incapable of solving the problem in a bounded time. However, imposing quality constraint might eliminate the part of the solution space to which $H_{\max}$ belongs. Thus, we cannot verify the formula in constrained diversity optimisation. Second, the Cplex solver is incapable of solving medium or large instances, even in unconstrained diversity optimisation, and there is no point in investigating tiny instances solely.  

Here, the experiments take place on all combinations of $\mu \in \{6,12,24\}$, $n \in \{5,10,15,20\}$ and $k \in \{2,3\}$. A time-bound of 24 hours is considered for the Cplex solver. The results are summarised in Table~\ref{tab:unconstrained_Cplex}. Note that the MIP formulation's objective function is to minimise $f_{\max}-f_{\min}$, while the EA uses the entropy value as the fitness function. 

In Table~\ref{tab:unconstrained_Cplex}, $O$ and $N$ represent the capability and incapability of the Cplex solver in converging to the global optimum within the time-bound, respectively. Table~\ref{tab:unconstrained_Cplex} indicates cases where Cplex cannot solve instances to the optimal value ($n = 20$, $\mu = 24$, and $k = 2$). Furthermore, it cannot find a feasible solution for the instances where $n \in \{10,15,20\}$, $\mu = 24$, and $k = 3$. This highlights the need for an efficient algorithm within a time-bound. More importantly, Table~\ref{tab:unconstrained_Cplex} shows that where the Cplex solver finds the optimal solution, the proposed EA converges to a population with the same entropy, which is consistent with the $H_{\max}$~\ref{eq:max}. The fact that both Cplex and the EA converged to $H_{\max}$ implies that Equation~\ref{eq:max} is correct. Therefore, we can use $H_{\max}$ as another termination criterion of the introduced EA and the baseline for further experimental investigation. Furthermore, the proposed EA converges in less than two minutes and less than a thousand iterations (cost evaluations) overall instances.     
\begin{table}
\caption{Comparison of the entropy of final populations obtained by Cplex and the EA (symbols $O$ and $N$ indicate whether Cplex converged within the given time-bound).}
\vspace{2mm}
\label{tab:unconstrained_Cplex}
\centering
\renewcommand{\tabcolsep}{12pt}
\renewcommand{\arraystretch}{1.2}
\begin{footnotesize}
\begin{tabular}{ccccccccc}
\toprule
& & \multicolumn{3}{c}{$\mathbf{n = 5}$} & 
\multicolumn{3}{c}{$\mathbf{n = 10}$}\\
\cmidrule(l{2pt}r{2pt}){3-5}
\cmidrule(l{2pt}r{2pt}){6-8}
$\mu$ & $k$ & \textbf{ENT} & \textbf{Cplex} & \textbf{$H_{\max}$} & \textbf{ENT} & \textbf{Cplex} & \textbf{$H_{\max}$} \\
\midrule
6 & 2 & $3.00$ & $3.00$ $(O)$ & $3.00$ & $4.44$ & $4.44$ $(O)$ & $4.44$\\
6 & 3 & $4.09$ & $4.09$ $(O)$ & $4.09$ & $4.79$ & $4.79$ $(O)$ & $4.79$\\
12 & 2 & $3.00$ & $3.00$ $(O)$ & $3.00$ & $4.48$ & $4.48$ $(O)$ & $4.48$\\
12 & 3 & $4.09$ & $4.09$ $(O)$ & $4.09$ & $5.48$ & $5.48$ $(O)$ & $5.48$\\
24 & 2 & $3.00$ & $3.00$ $(O)$ & $3.00$ & $4.50$ & $4.50$ $(O)$ & $4.50$\\
24 & 3 & $4.09$ & $4.09$ $(O)$ & $4.09$ & $6.17$ & - & $6.17$\\
\midrule
& & \multicolumn{3}{c}{$\mathbf{n = 15}$} &
\multicolumn{3}{c}{$\mathbf{n = 20}$} \\
\cmidrule(l{2pt}r{2pt}){3-5}
\cmidrule(l{2pt}r{2pt}){6-8}
$\mu$ & $k$ & \textbf{ENT} & \textbf{Cplex} & \textbf{$H_{\max}$} & \textbf{ENT} & \textbf{Cplex} & \textbf{$H_{\max}$} \\
\midrule
6 & 2 & $5.19$ & $5.19$ $(O)$ & $5.19$ & $5.48$ & $5.48$ $(O)$ & $5.48$ \\
6 & 3 & $5.19$ & $5.19$ $(O)$ & $5.19$ & $5.48$ & $5.48$ $(O)$ & $5.48$ \\
12 & 2 & $5.31$ & $5.31$ $(O)$ & $5.31$ & $5.88$ & $5.88$ $(O)$ & $5.88$ \\
12 & 3 & $5.89$ & $5.89$ $(O)$ & $5.89$ & $6.17$ & $6.17$ $(O)$ & $6.17$ \\
24 & 2 & $5.34$ & $5.33$ $(O)$ & $5.34$ & $5.92$ & $5.91$ $(N)$ & $5.92$ \\
24 & 3 & $6.58$ & - & $6.58$ & $6.87$ & - & $6.87$ \\
\bottomrule

\end{tabular}
\end{footnotesize}
\end{table}
\subsection{Unconstrained Diversity Optimisation} \label{subsec:Uncon_Res}
We first compare classic 2-OPT and biased 2-OPT.  We claimed that biased 2-OPT is more likely to generate offspring contributing to the population's entropy, while classic 2-OPT may perform better in generating tours satisfying the quality constraint. Since no quality constraints are imposed in this subsection, biased 2-OPT is expected to outperform its counterpart. For comparison, we conduct experiments on a complete graph with 100 nodes, and consider $k=2$ and $\mu \in \{25, 125, 250\}$.   

Figure~\ref{fig:Un2opt} compares the convergence pace of classic 2-OPT and biased 2-OPT. The entropy value is shown on the $y$-axis, whereby the $x$-axis represents the number of cost evaluations (iterations). Note that diversity scores shown on the figure are normalised by using Equations~\ref{eq:max} and \ref{eq:min}. Figure \ref{fig:Un2opt} indicates that both operators eventually converge to $H_{\max}$ in most cases. However, biased 2-OPT is faster than the classic 2-OPT.

Figure 3 compares the number of cost evaluations required to converge to $H_{\max}$ in the introduced EA using classic and biased 2-OPT over ten runs. One can observe that the number of required cost evaluations is significantly higher for classic 2-OPT. Biased $2$-OPT, for example, requires around $2,350$ evaluations on average when $\mu = 25$. On the other hand, the figure is around $14\,000$ for classic $2$-OPT. Moreover, none of the operators converges to $H_{\max}$ within the limit of $100\,000$ cost evaluations where $\mu = 250$. In this case, the mean of the entropy value of biased and classic 2-OPT are 9.1993 and 9.1983, respectively, while $H_{\max}$ is equal to 9.1994.  
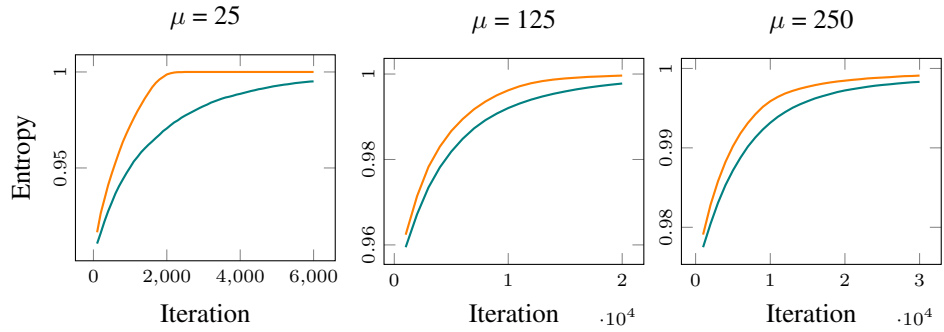
\begin{figure}
\centering
\input{pic5}
\caption{Comparison between convergence pace of biased 2-OPT (orange) and classic 2-OPT (green).}
\label{fig:Un2opt}
\end{figure}

\begin{figure}
    \centering
    \includegraphics[width=0.7\columnwidth,scale=1]{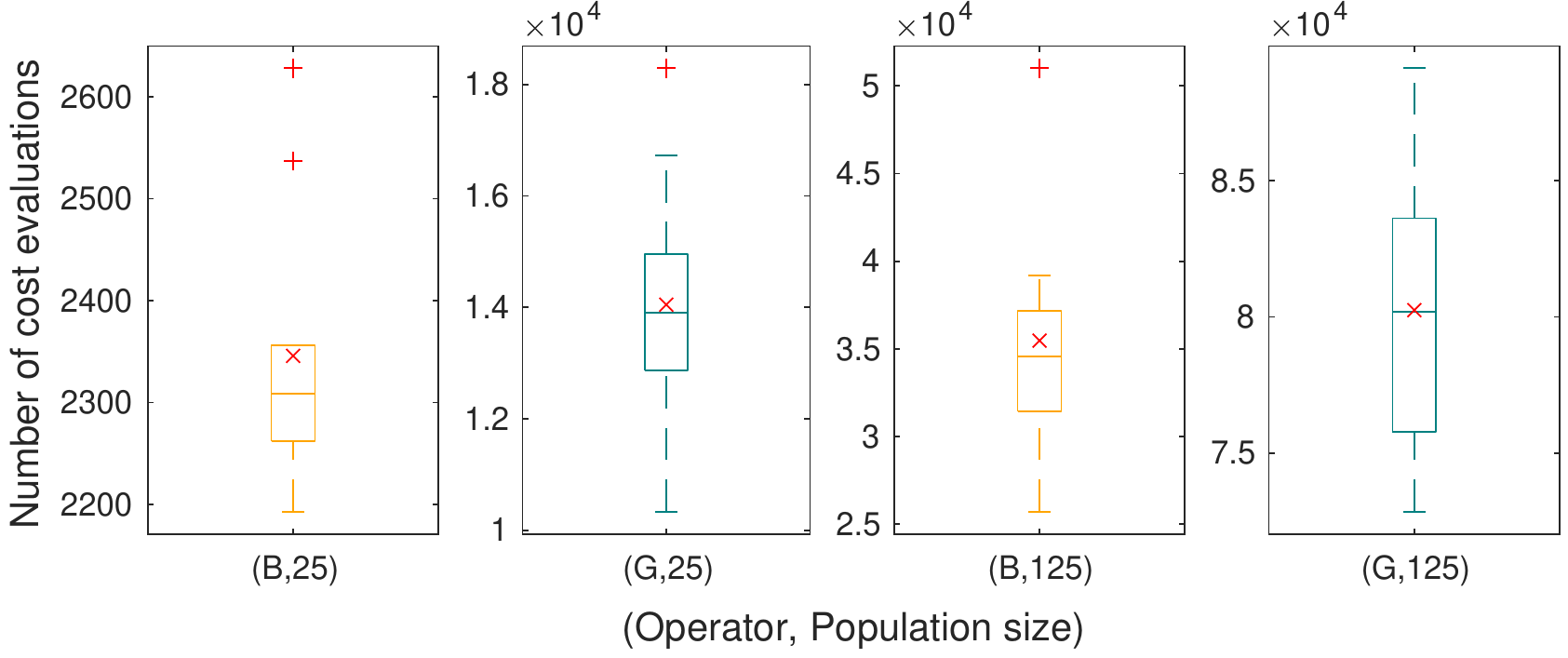}
    \caption{Differences in the required number of cost evaluations to reaching $H_{\max}$ for the biased 2-OPT (orange) and the classic 2-OPT (green).
    }
    \label{fig:operators}
\end{figure}

Next, we provide a comprehensive comparison between the proposed EA and EAs based on ED and PD proposed by~\citet{viet2020evolving}. We conduct experiments on all combinations of $n \in \{50, 100\}$, $\mu \in \{12,20,50,100,500,1\,000\}$ and $k \in \{2,3,4\}$. The termination criteria are reaching either the entropy value of $H_{\max}$ or the limitation of $100\,000$ cost evaluations. Note that the EAs based on ED and PD compare the offspring to the entire population, requiring considerably more diversity evaluations per generated offspring.
Table~\ref{tabel:Uncon} compares the entropy of the final population obtained from the algorithms. Here, we solely use biased 2-OPT due to its efficiency in unconstrained diversity optimisation. The results show that the proposed algorithm outperforms the algorithms based on ED and PD over large populations and long segments, (\ie $\mu \in \{500, 1000\}$ and $k \in \{3,4\}$). In the case $n=50$, $k=3$ and $\mu=1000$, for instance, the introduced EA scores $11.35$ entropy value while the algorithms based on ED and PD achieve $10.73$ and $11.03$, respectively.   
\begin{table*}[htb]
\caption{Comparison between the high-order entropy values of the final populations of the introduced EA and ones based on ED and PD. Stat shows the results of a Kruskal-Wallis test at significance level of $95\%$ with Bonferroni correction. $X^+$ means the median of the measure is better than the one for variant $X$, $X^-$ means it is worse and $X^*$ indicates no significant difference.}
\label{tabel:Uncon}
\centering
\renewcommand{\tabcolsep}{2pt}
\renewcommand{\arraystretch}{1.2}
\begin{footnotesize}
\begin{tabular}{cccccccccccccccccc}
\toprule
& & 
\multicolumn{8}{c}{\textbf{$n=50$}} & 
\multicolumn{8}{c}{\textbf{$n=100$}} \\
\cmidrule(l{2pt}r{2pt}){3-10}
\cmidrule(l{2pt}r{2pt}){11-18}
& & \multicolumn{2}{c}{ENT~(1)} & \multicolumn{2}{c}{ED~(2)} & \multicolumn{2}{c}{PD~(3)} & \multicolumn{2}{c}{Range} &
\multicolumn{2}{c}{ENTB~(1)} & \multicolumn{2}{c}{ED~(2)} & \multicolumn{2}{c}{PD~(3)} & \multicolumn{2}{c}{Range}\\
\cmidrule(l{2pt}r{2pt}){3-4}
\cmidrule(l{2pt}r{2pt}){5-6}
\cmidrule(l{2pt}r{2pt}){7-8}
\cmidrule(l{2pt}r{2pt}){9-10}
\cmidrule(l{2pt}r{2pt}){11-12}
\cmidrule(l{2pt}r{2pt}){13-14}
\cmidrule(l{2pt}r{2pt}){15-16}
\cmidrule(l{2pt}r{2pt}){17-18}
$\mu$ & $k$ & \textbf{mean} & \textbf{stat} & \textbf{mean} & \textbf{stat} & \textbf{mean} & \textbf{stat} & \textbf{$H_{\min}$} & \textbf{$H_{\max}$} & \textbf{mean} & \textbf{stat} & \textbf{mean} & \textbf{stat} & \textbf{mean} & \textbf{stat} & \textbf{$H_{\min}$} & \textbf{$H_{\max}$} \\
\midrule
12 & 2 & 7.09 & $2^{*} 3^{*}$ & 7.09 & $1^{*} 3^{*}$ & 7.09 & $1^{*} 2^{*}$ & 4.6052 & 7.0901 & 7.78 & $2^{*} 3^{*}$ & 7.78 & $1^{*} 3^{*}$ & 7.78 & $1^{*} 2^{*}$ & 5.2983 & 7.7832 \\
12 & 3 & 7.09 & $2^{*} 3^{*}$ & 7.09 & $1^{*} 3^{*}$ & 7.09 & $1^{*} 2^{*}$ & 4.6052 & 7.0901 & 7.78 & $2^{*} 3^{*}$ & 7.78 & $1^{*} 3^{*}$ & 7.78 & $1^{*} 2^{*}$ & 5.2983 & 7.7832 \\
12 & 4 & 7.09 & $2^{*} 3^{*}$ & 7.09 & $1^{*} 3^{*}$ & 7.09 & $1^{*} 2^{*}$ & 4.6052 & 7.0901 & 7.78 & $2^{*} 3^{*}$ & 7.78 & $1^{*} 3^{*}$ & 7.78 & $1^{*} 2^{*}$ & 5.2983 & 7.7832 \\
\cmidrule{1-18}
20 & 2 & 7.60 & $2^{*} 3^{*}$ & 7.60 & $1^{*} 3^{*}$ & 7.60 & $1^{*} 2^{*}$ & 4.6052 & 7.6006 & 8.29 & $2^{*} 3^{*}$ & 8.29 & $1^{*} 3^{*}$ & 8.29 & $1^{*} 2^{*}$ & 5.2983 & 8.2940 \\
20 & 3 & 7.60 & $2^{*} 3^{*}$ & 7.60 & $1^{*} 3^{*}$ & 7.60 & $1^{*} 2^{*}$ & 4.6052 & 7.6006 & 8.29 & $2^{*} 3^{*}$ & 8.29 & $1^{*} 3^{*}$ & 8.29 & $1^{*} 2^{*}$ & 5.2983 & 8.2940 \\
20 & 3 & 7.60 & $2^{*} 3^{*}$ & 7.60 & $1^{*} 3^{*}$ & 7.60 & $1^{*} 2^{*}$ & 4.6052 & 7.6006 & 8.29 & $2^{*} 3^{*}$ & 8.29 & $1^{*} 3^{*}$ & 8.29 & $1^{*} 2^{*}$ & 5.2983 & 8.2940 \\
\cmidrule{1-18}
50 & 2 & 7.80 & $2^{*} 3^{*}$ & 7.80 & $1^{*} 3^{*}$ & 7.80 & $1^{*} 2^{*}$ & 4.6052 & 7.7997 & 9.17 & $2^{+} 3^{+}$ & 9.14 & $1^{-} 3^{*}$ & 9.13 & $1^{-} 3^{*}$ & 5.2983 & 9.1965 \\
50 & 3 & \hl{\textbf{8.52}} & $2^{+} 3^{+}$ & 8.51 & $1^{-} 3^{*}$ & 8.51 & $1^{-} 2^{*}$ & 4.6052 & 8.5172 & 9.21 & $ 2^{*} 3^{*}$ & 9.21 & $1^{*} 3^{*}$ & 9.21 & $1^{*} 2^{*}$ & 5.2983 & 9.2103 \\
50 & 4 & 8.52 & $2^{*} 3^{*}$ & 8.52 & $1^{*} 3^{*}$ & 8.52 & $1^{*} 2^{*}$ & 4.6052 & 8.5172 & 9.21 & $2^{*} 3^{*}$ & 9.21 & $1^{*} 3^{*}$ & 9.21 & $1^{*} 2^{*}$ & 5.2983 & 9.2103 \\
\cmidrule{1-18}
100 & 2 & 7.80 & $2^{*} 3^{+}$ & 7.80 & $1^{*} 3^{+}$ & 7.79 & $1^{-} 2^{-}$ & 4.6052 & 7.8017 & 9.19 & $2^{+} 3^{+}$ & 9.18 & $1^{-} 3^{*}$ & 9.16 & $1^{-} 2^{*}$ & 5.2983 & 9.1982 \\
100 & 3 & \hl{\textbf{9.21}} & $2^{+} 3^{+}$ & 9.18 & $1^{-} 3^{*}$ & 9.19 & $1^{-} 2^{*}$ & 4.6052 & 9.2103 & 9.90 & $2^{+} 3^{+}$ & 9.90 & $1^{-} 3^{*}$ & 9.90 & $1^{-} 2^{*}$ & 5.2983 & 9.9035 \\
100 & 4 & 9.21 & $2^{+} 3^{*}$ & 9.21 & $1^{-} 3^{-}$ & 9.21 & $1^{*} 2^{+}$ & 4.6052 & 9.2103 & 9.90 & $2^{+} 3^{*}$ & 9.90 & $1^{-} 3^{-}$ & 9.90 & $1^{*} 2^{+}$ & 5.2983 & 9.9035 \\
\cmidrule{1-18}
500 & 2 & 7.80 & $2^{*} 3^{+}$ & 7.80 & $1^{*} 3^{+}$ & 7.80 & $1^{-} 2^{-}$ & 4.6052 & 7.8036 & 9.20 & $2^{+} 3^{+}$ & 9.20 & $1^{-} 3^{+}$ & 9.16 & $1^{-} 2^{-}$ & 5.2983 & 9.1999 \\
500 & 3 & \hl{\textbf{10.82}} & $2^{+} 3^{+}$ & 10.45 & $1^{-} 3^{*}$ & 10.60 & $1^{-} 2^{*}$ & 4.6052 & 10.8198 & \hl{\textbf{11.51}} & $2^{+} 3^{+}$ & 11.34 & $1^{-} 3^{*}$ & 11.45 & $1^{-} 2^{*}$ & 5.2983 & 11.5129 \\
500 & 4 & 10.82 & $2^{+} 3^{+}$ & 10.76 & $1^{-} 3^{*}$ & 10.82 & $1^{-} 2^{*}$ & 4.6052 & 10.8198 & 11.51 & $2^{+} 3^{+}$ & 11.47 & $1^{-} 3^{*}$ & 11.51 & $1^{-} 2^{*}$ & 5.2983 & 11.5129 \\
\cmidrule{1-18}
1000 & 2 & 7.80 & $2^{*} 3^{+}$ & 7.80 & $1^{*} 3^{+}$ & 7.79 & $1^{-} 2^{-}$ & 4.6052 & 7.8038 & 9.20 & $3^{*} 4^{+}$ & 9.16 & $1^{*} 3^{+}$ & 9.01 & $1^{-} 2^{-}$ & 5.2983 & 9.2001\\
1000 & 3 & \hl{\textbf{11.35}} & $2^{+} 3^{+}$ & 10.73 & $1^{-} 3^{*}$ & 11.03 & $1^{-} 2^{*}$ & 4.6052 & 11.5129 & \hl{\textbf{12.16}} & $2^{+} 3^{+}$ & 11.33 & $1^{-} 3^{*}$ & 11.89 & $1^{-} 3^{*}$ & 5.2983 & 12.2061 \\
1000 & 4 & \hl{\textbf{11.52}} & $2^{+} 3^{+}$ & 11.30 & $1^{-} 3^{*}$ & 11.50 & $1^{-} 2^{*}$ & 4.6052 & 11.5129 & \hl{\textbf{12.21}} & $3^{+} 4^{+}$ & 10.36 & $1^{-} 3^{-}$ & 11.90 & $1^{-} 2^{+}$ & 5.2983 & 12.2061\\
 
\bottomrule
\end{tabular}
\end{footnotesize}
\end{table*}

\subsection{Constrained Diversity Optimisation} \label{subsec:con_Res}

In constrained diversity optimisation, the performance of classic 2-OPT and biased 2-OPT is strongly correlated to the threshold; the wider the threshold, the better the performance of biased 2-OPT, and vice versa. Therefore, we used both operators in this subsection (see Algorithm \ref{alg:ea}). Here, the experiments are conducted on eil51, eil76, and eil101 from the TSPlib, \cite{Reinelt91tsplib} where a threshold of $\alpha = 5\%$ is considered. Moreover, the limit of cost evaluations increases to $300\,000$ due to the imposition of the quality constraint. 

\begin{table*}[!t]
\caption{Comparison between the high-order entropy values of final populations of the introduced EA and EAs based on ED and PD on TSPlib instances eil51, eil76 and eil101 (threshold is equal to $\alpha=0.05$). Tests and notations are in line with Table \ref{tabel:Uncon}.}
\label{tab:constrained}
\centering
\renewcommand{\tabcolsep}{0.7pt}
\renewcommand{\arraystretch}{1.2}
\begin{footnotesize}
\begin{tabular}{cccccccccccccccccccc}
\toprule
& & \multicolumn{6}{c}{\textbf{eil51~($H_{\min} = 4.6250$)}} & 
\multicolumn{6}{c}{\textbf{eil76~($H_{\min} = 5.0239$)}} & 
\multicolumn{6}{c}{\textbf{eil101~($H_{\min} = 5.3083$)}} \\
\cmidrule(l{2pt}r{2pt}){3-8}
\cmidrule(l{2pt}r{2pt}){9-14}
\cmidrule(l{2pt}r{2pt}){15-20}
& & \multicolumn{2}{c}{ENT~(1)} &
\multicolumn{2}{c}{ED~(2)} & \multicolumn{2}{c}{PD~(3)} &
\multicolumn{2}{c}{ENT~(1)} &
\multicolumn{2}{c}{ED~(2)} & \multicolumn{2}{c}{PD~(3)} & 
\multicolumn{2}{c}{ENT~(1)} &
\multicolumn{2}{c}{ED~(2)} & \multicolumn{2}{c}{PD~(3)} \\
\cmidrule(l{2pt}r{2pt}){3-4}
\cmidrule(l{2pt}r{2pt}){5-6}
\cmidrule(l{2pt}r{2pt}){7-8}
\cmidrule(l{2pt}r{2pt}){9-10}
\cmidrule(l{2pt}r{2pt}){11-12}
\cmidrule(l{2pt}r{2pt}){13-14}
\cmidrule(l{2pt}r{2pt}){15-16}
\cmidrule(l{2pt}r{2pt}){17-18}
\cmidrule(l{2pt}r{2pt}){19-20}
$\mu$ & $k$ & \textbf{mean} & \textbf{stat} & \textbf{mean} & \textbf{stat} & \textbf{mean} & \textbf{stat} & \textbf{mean} & \textbf{stat} & \textbf{mean} & \textbf{stat} & \textbf{mean} & \textbf{stat} & \textbf{mean} & \textbf{stat} & \textbf{mean} & \textbf{stat} & \textbf{mean} & \textbf{stat} \\
\midrule
12 & 2 &  \hl{\textbf{5.1133}} & $2^{+} 3^{+}$ & 5.0586 & $1^{-} 3^{+}$ & 5.0381 & $1^{-} 2^{+}$ & \hl{\textbf{5.4617}} & $2^{+} 3^{+}$ & 5.4047 & $1^{-} 3^{*}$ & 5.3872 & $1^{-} 2^{*}$ & \hl{\textbf{5.8137}} & $2^{+} 3^{+}$ & 5.7674 & $1^{-} 3^{*}$ & 5.7580 & $1^{-} 2^{*}$ \\
12 & 3 & \hl{\textbf{5.5648}} & $2^{*} 3^{+}$ & 5.4964 & $1^{*} 3^{+}$ & 5.4216 & $1^{-} 2^{-}$ & \hl{\textbf{5.8517}} & $2^{*} 3^{+}$ & 5.7699 & $1^{-} 3^{*}$ & 5.6977 & $1^{-} 2^{*}$ & \hl{\textbf{6.2213}} & $2^{*} 3^{+}$ & 6.1784 & $1^{-} 3^{+}$ & 6.1275 & $1^{-} 2^{-}$ \\
12 & 4 & \hl{\textbf{5.7640}} & $2^{+} 3^{+}$ & 5.6764 & $1^{*} 3^{*}$ & 5.6043 & $1^{-} 2^{*}$ & \hl{\textbf{6.0499}} & $2^{+} 3^{+}$ & 5.9346 & $1^{-} 3^{*}$ & 5.8546 & $1^{-} 2^{*}$ & \hl{\textbf{6.4660}} & $2^{+} 3^{+}$ & 6.3742 & $1^{-} 3^{*}$ & 6.3058 & $1^{-} 2^{*}$ \\
\cmidrule{1-20}
20 & 2 & \hl{\textbf{5.1354}} & $2^{+} 3^{+}$ & 5.0543 & $1^{-} 3^{*}$ & 5.0687 & $1^{-} 2^{*}$ & \hl{\textbf{5.4843}} & $2^{+} 3^{+}$ & 5.4205 & $1^{-} 3^{*}$ & 5.4241 & $1^{-} 2^{*}$ & \hl{\textbf{5.8232}} & $2^{+} 3^{+}$ & 5.7961 & $1^{-} 3^{*}$ & 5.7822 & $1^{-} 2^{*}$ \\
20 & 3 & \hl{\textbf{5.6557}} & $2^{+} 3^{+}$ & 5.4943 & $1^{-} 3^{*}$ & 5.5157 & $1^{-} 2^{*}$ & \hl{\textbf{5.9351}} & $2^{+} 3^{+}$ & 5.7911 & $1^{-} 3^{*}$ & 6.7810 & $1^{-} 2^{*}$ & \hl{\textbf{6.3098}} & $2^{+} 3^{+}$ & 6.1778 & $1^{-} 3^{*}$ & 6.1812 & $1^{-} 2^{*}$ \\
20 & 4 & \hl{\textbf{5.9247}} & $2^{+} 3^{+}$ & 5.6846 & $1^{-} 2^{*}$ & 5.7386 & $1^{-} 2^{*}$ & \hl{\textbf{6.1831}} & $2^{+} 3^{+}$ & 5.9656 & $1^{-} 3^{*}$ & 5.9623 & $1^{-} 2^{*}$ & \hl{\textbf{6.5566}} & $2^{+} 2^{+}$ & 6.3810 & $1^{-} 3^{*}$ & 6.3834 & $1^{-} 2^{*}$ \\
\cmidrule{1-20}
50 & 2 & \hl{\textbf{5.1704}} & $2^{+} 3^{+}$ & 5.0618 & $1^{-} 3^{-}$ & 5.1017 & $1^{-} 2^{+}$ & \hl{\textbf{5.5015}} & $2^{+} 3^{+}$ & 5.4194 & $1^{-} 3^{-}$ & 4.4454 & $1^{-} 2^{+}$ & \hl{\textbf{5.8262}} & $2^{+} 3^{+}$ & 5.7607 & $1^{-} 3^{-}$ & 5.7938 & $1^{-} 2^{+}$ \\
50 & 3 & \hl{\textbf{5.7371}} & $2^{+} 3^{+}$ & 5.5087 & $1^{-} 3^{-}$ & 5.6150 & $1^{-} 2^{+}$ & \hl{\textbf{5.9961}} & $2^{+} 3^{+}$ & 5.7861 & $1^{-} 3^{-}$ & 5.8497 & $1^{-} 2^{+}$ & \hl{\textbf{6.3594}} & $2^{+} 3^{+}$ & 6.1816 & $1^{-} 3^{-}$ & 6.2370 & $1^{-} 2^{+}$ \\
50 & 4 & \hl{\textbf{6.0927}} & $2^{+} 3^{+}$ & 5.7123 & $1^{-} 3^{-}$ & 5.8982 & $1^{-} 2^{+}$ & \hl{\textbf{6.2776}} & $2^{+} 3^{+}$ & 5.9674 & $1^{-} 3^{*}$ & 6.0776 & $1^{-} 2^{*}$ & \hl{\textbf{6.6490}} & $2^{+} 3^{+}$ & 6.3997 & $1^{-} 3^{-}$ & 6.4858 & $1^{-} 2^{+}$ \\
\cmidrule{1-20}
100 & 2 &\hl{\textbf{ 5.1683}} & $2^{+} 3^{+}$ & 5.0623 & $1^{-} 3^{-}$ & 5.1033 & $1^{-} 2^{+}$ & \hl{\textbf{5.4911}} & $2^{+} 3^{+}$ & 5.4227 & $1^{-} 3^{-}$ & 5.4464 & $1^{-} 2^{+}$ & \hl{\textbf{5.7980}} & $2^{+} 3^{+}$ & 5.7569 & $1^{-} 3^{-}$ & 5.7804 & $1^{-} 2^{+}$ \\
100 & 3 & \hl{\textbf{5.7503}} & $2^{+} 3^{+}$ & 5.5175 & $1^{-} 2^{-}$ & 5.6452 & $1^{-} 2^{+}$ & \hl{\textbf{5.9870}} & $2^{+} 3^{+}$ & 5.8120 & $1^{-} 3^{-}$ & 5.8658 & $1^{-} 2^{+}$ & \hl{\textbf{6.2890}} & $2^{+} 3^{+}$ & 6.1838 & $1^{-} 3^{-}$ & 6.2291 & $1^{-} 2^{+}$ \\
100 & 4 & \hl{\textbf{6.1436}} & $2^{+} 3^{+}$ & 5.7319 & $1^{-} 3^{-}$ & 5.9646 & $1^{-} 2^{+}$ & \hl{\textbf{6.3027}} & $2^{+} 3^{+}$ & 6.0127 & $1^{-} 3^{-}$ & 6.1098 & $1^{-} 2^{+}$ & \hl{\textbf{6.6246}} & $2^{+} 3^{+}$ & 6.4137 & $1^{-} 3^{-}$ & 6.4938 & $1^{-} 2^{+}$ \\
\cmidrule{1-20}
500 & 2 & \hl{\textbf{5.1203}} & $2^{+} 3^{+}$ & 5.0396 & $1^{-} 3^{-}$ & 5.0815 & $1^{-} 2^{+}$ & \hl{\textbf{5.4320}} & $2^{+} 3^{*}$ & 5.4013 & $1^{-} 3^{-}$ & 5.4244 & $1^{*} 2^{+}$ & 5.7070 & $2^{*} 3^{-}$ & 5.7111 & $1^{*} 3^{-}$ & \hl{\textbf{5.7377}} & $1^{+} 2^{+}$ \\
500 & 3 & \hl{\textbf{5.6794}} & $2^{+} 3^{+}$ & 5.5131 & $1^{-} 3^{-}$ & 5.6359 & $1^{-} 2^{+}$ & \hl{\textbf{5.8876}} & $2^{+} 3^{+}$ & 5.8077 & $1^{-} 3^{-}$ & 5.8653 & $1^{-} 2^{+}$ & 6.1379 & $2^{*} 3^{-}$ & 6.1180 & $1^{*} 3^{-}$ & \hl{\textbf{6.1808}} & $1^{+} 2^{+}$ \\
500 & 4 & \hl{\textbf{6.0864}} & $2^{+} 3^{+}$ & 5.7660 & $1^{-} 3^{-}$ & 5.9991 & $1^{-} 2^{+}$ & \hl{\textbf{6.2218}} & $2^{+} 3^{+}$ & 6.0399 & $1^{-} 3^{-}$ & 6.1469 & $1^{-} 2^{+}$ & \hl{\textbf{6.5770}} & $2^{+} 3^{*}$ & 6.3648 & $1^{-} 3^{-}$ & 6.4616 & $1^{*} 2^{+}$ \\
\cmidrule{1-20}
1000 & 2 & \hl{\textbf{5.0909}} & $2^{+} 3^{+}$ & 5.0187 & $1^{-} 3^{-}$ & 5.0585 & $1^{-} 2^{+}$ & 5.4074 & $2^{+} 3^{*} $ & 5.3811 & $1^{-} 3^{-}$ & \hl{\textbf{5.6926}} & $1^{*} 2^{+}$ & 5.7194 & $2^{*} 3^{-}$ & 5.6933 & $1^{*} 3^{*}$ & 5.7194 & $1^{+} 2^{*}$ \\
1000 & 3 & \hl{\textbf{5.6291}} & $2^{+} 3^{+}$ & 5.4760 & $1^{-} 3^{-}$ & 5.5943 & $1^{*} 2^{+}$ & \hl{\textbf{5.8442}} & $2^{+} 3^{*}$ & 5.7712 & $1^{-} 3^{-}$ & 5.8333 & $1^{*} 2^{+}$ & 6.0987 & $2^{*} 3^{-}$ & 6.0891 & $1^{*} 3^{-}$ & \hl{\textbf{6.1495}} & $1^{*} 2^{+}$ \\
1000 & 4 & \hl{\textbf{6.0238}} & $2^{+} 3^{+}$ & 5.7357 & $1^{-} 3^{-}$ & 5.9498 & $1^{*} 2^{+}$ & \hl{\textbf{6.1771}} & $2^{+} 3^{+}$ & 6.0039 & $1^{-} 3^{-}$ & 6.1116 & $1^{-} 2^{+}$ & \hl{\textbf{6.4372}} & $2^{+} 3^{*}$ & 6.3372 & $1^{-} 3^{-}$ & 6.4348 & $1^{*} 2^{+}$ \\
\bottomrule
\end{tabular}
\end{footnotesize}
\end{table*}

In line with the unconstrained diversity optimisation, Table~\ref{tab:constrained} compares the entropy value of the final population obtained from the introduced algorithm and the algorithms based on ED and PD in~\cite{viet2020evolving}. Table~\ref{tab:constrained} indicates that the introduced EA outperforms the ones based on ED and PD in most instances. The algorithm based on PD has achieved a better entropy value over only four cases. Given that all these four cases are among the largest ones, a possible reason could be differences in the algorithms' survival selection resulting in slower convergence of the introduced EA than the others in terms of cost evaluations. However, the smaller instances show that if the number of cost evaluations is sufficient, the introduced EA is likely to outperform the others. We conduct another experiment summarised in Figure~\ref{fig:itera} to more elaborate on this matter.  
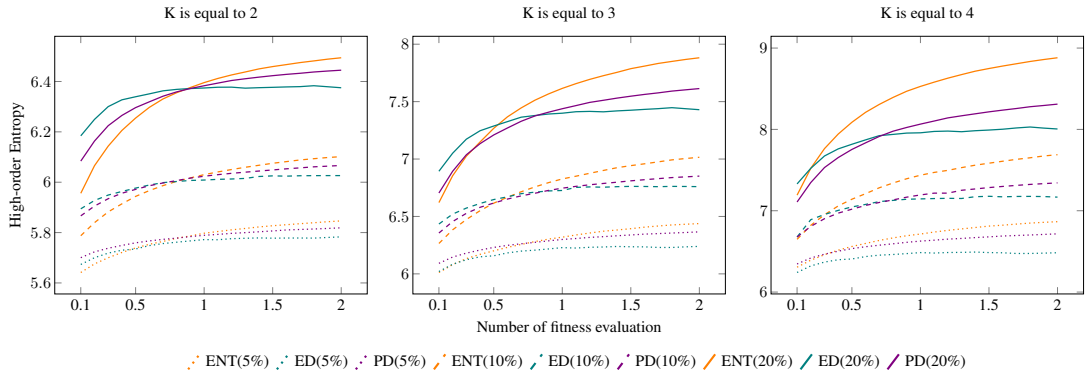
\begin{figure*}
\centering
\input{pic4}

\caption{Impact of the number of fitness evaluations on the algorithms on eil101. The percentage values show the allowed threshold.}
\label{fig:itera}
\end{figure*}
In Figure~\ref{fig:itera}, the number of cost evaluations is shown on the $x$-axis, while the $y$-axis presents the final population's entropy. Figure \ref{fig:itera} shows the results on eil101, $\mu=500$, $k \in \{2,3,4\}$ and $\alpha \in \{0.05,0.1,0.2\}$. Since we observed the same pattern for the other cases, the figure is contented for the sake of brevity. Figure~\ref{fig:itera} indicates that if the number of cost evaluations is deficient, the introduced EA results in a lower entropy value than the other algorithms. Nevertheless, it always converges to a higher entropy value. This pattern can be observed for all nine combinations. As the value of $k$ rises, the introduced EA surpasses the other two in less number of cost evaluations. In comparing ED and PD, the algorithm using ED converges faster but at a lower entropy.

Figure \ref{fig:edge_overlays} shows the edges used in the sets of $125$ tours obtained from the introduced EA in constrained ($\alpha \in \{0, 0.05, 0.5\}$) and unconstrained diversity optimisation on eil101. The figure clearly highlights the proportional relationship between $\alpha$ and the diversity of the population. Figure \ref{fig:edge_overlays} also depicts the differences in a population with the entropy value of $H_{\min}$ (first plot on the top) with a population with $H_{\max}$ entropy (the second plot on the bottom). As the population's entropy increases, the number of incorporated edges (segments) rises while the frequency of edges inclines. 

 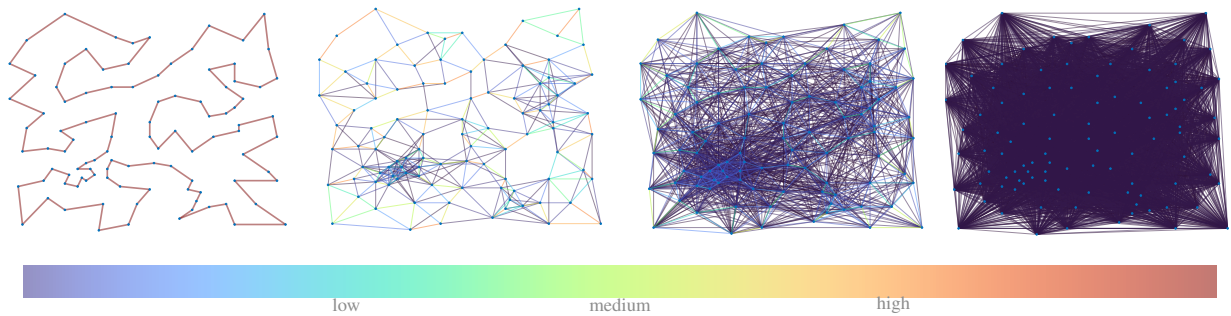
\begin{figure}[ht]
  \centering
  \includegraphics[width=.23\columnwidth]{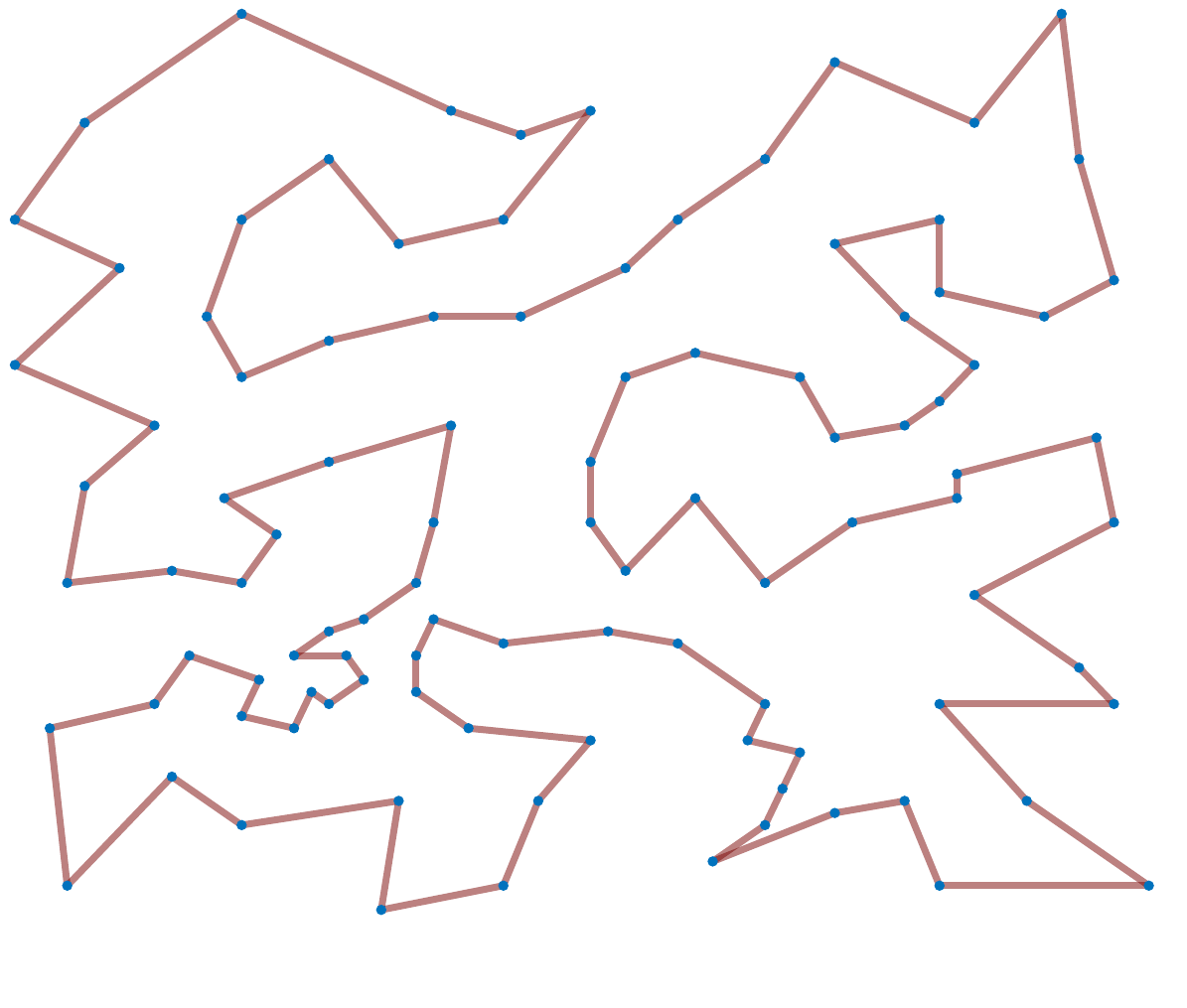}
  \hskip5pt
   \includegraphics[width=.23\columnwidth]{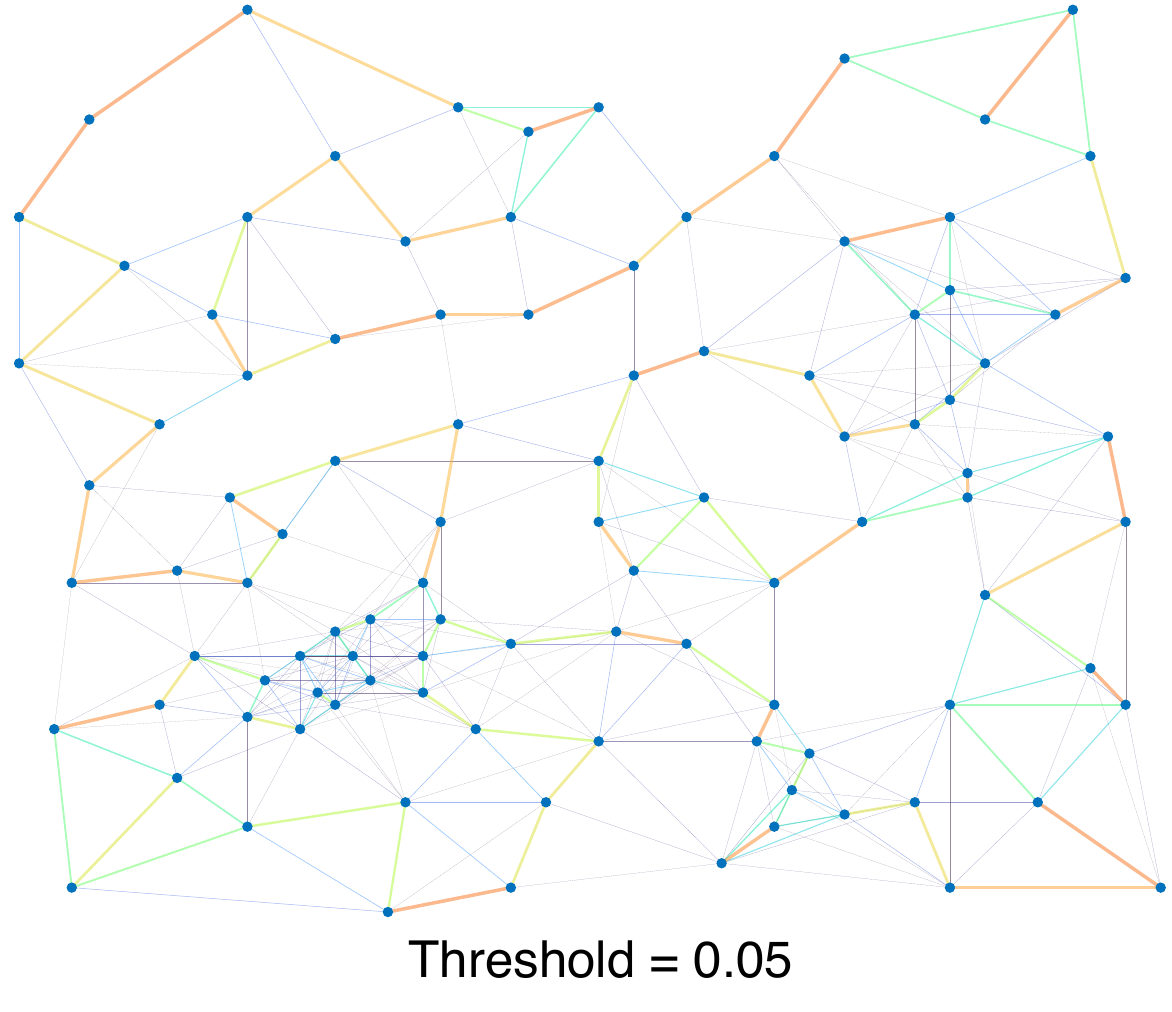}
   \hskip10pt
   \includegraphics[width=.23\columnwidth]{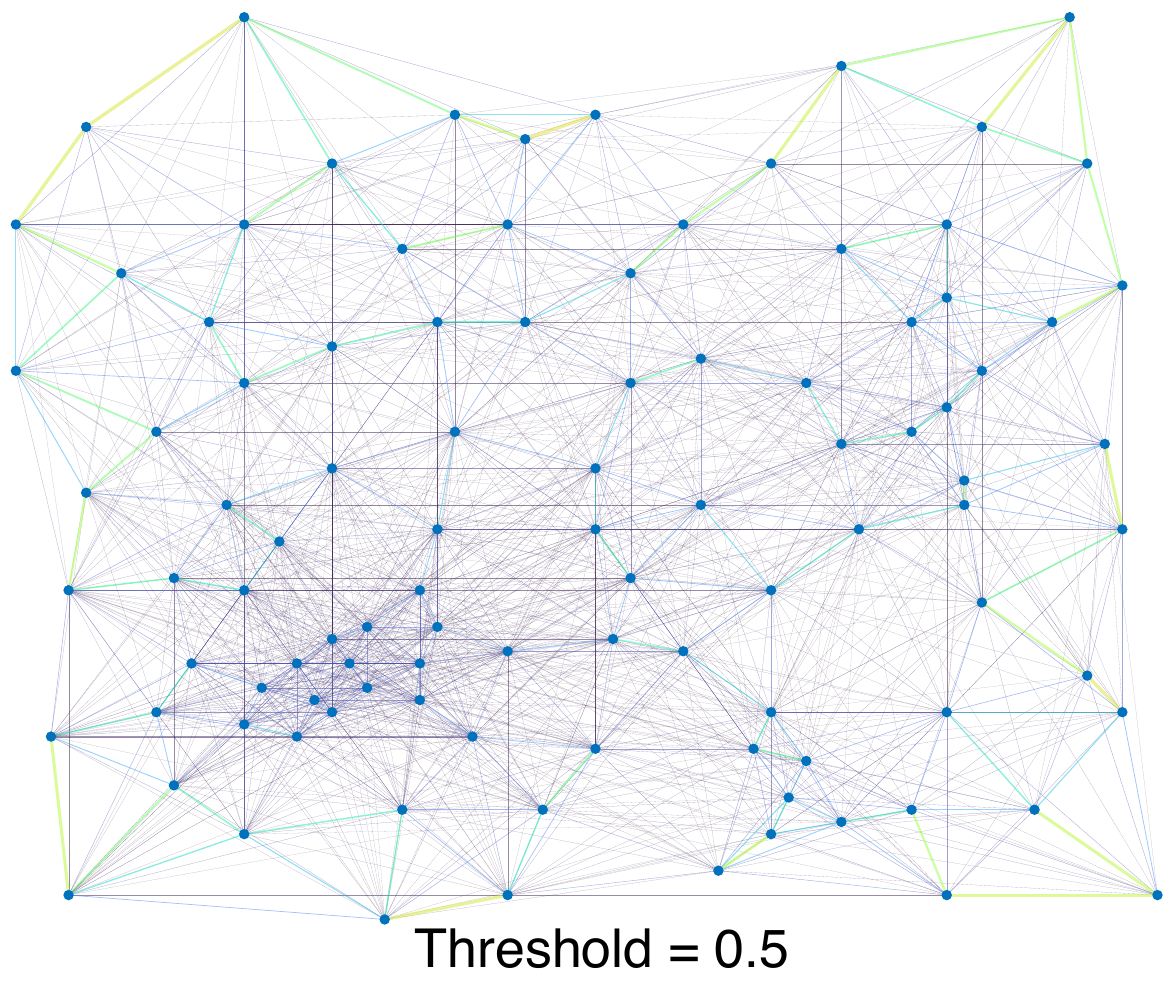}
   \hskip5pt
  \includegraphics[width=.23\columnwidth]{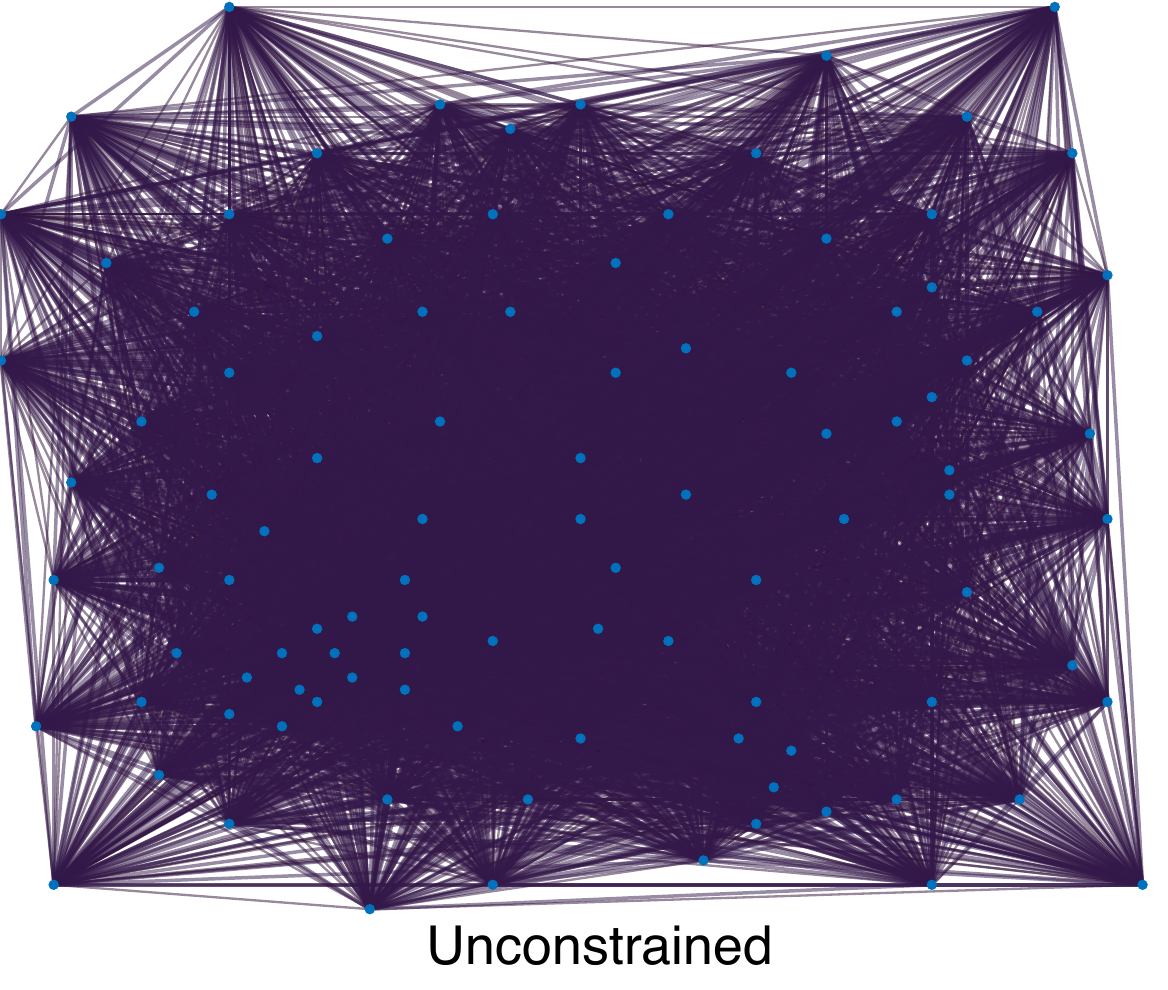}
 \begin{tikzpicture}
 \node[opacity=.55] (legend) at (0,0) {\includegraphics[width=0.95\columnwidth,trim=20pt 40pt 10pt 0,clip]{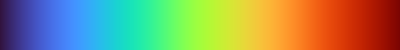}};
  \node (low) at (-3.6,-0.3) {\scriptsize{\textcolor{gray!90}{low}}};
  \node (medium) at (0,-0.3) {\scriptsize{\textcolor{gray!90}{medium}}};
  \node (high) at (3.6,-0.3) {\scriptsize{\textcolor{gray!90}{high}}};
  \end{tikzpicture}
 \caption{Overlay of the edges (coloured based on their frequency) incorporated into the population of the introduced EA on eil101 where $\alpha$ increases from $0$ to $+\infty$.}
 \label{fig:edge_overlays}
 \end{figure}

\section{Conclusion} \label{sec:con}

In the context of EDO, we aim to evolve diverse sets of solutions meeting minimal quality criteria. EDO has been rarely considered for classical combinatorial optimisation problems so far. We adopted a new diversity measure based on high-order entropy to maximise the diversity of a population of TSP solutions. The diversity measure allows equalising the share of segments of multiple nodes, whereas previously proposed diversity measures by~\cite{viet2020evolving} in the TSP context focus on the frequency of single edges in the population. 
We show theoretical properties that a maximally/minimally diverse set of solutions has to fulfill. Furthermore, we study the effects of the high-order entropy measure embedded into a simple population-based evolutionary algorithm experimentally. This algorithm uses different versions of 2-OPT mutations partially biased towards favouring high-frequency segments in TSP tours.
Our results in the unconstrained setting without quality restriction and the constrained setting on TSPlib instances show the superiority of the proposed approach if the number of cost evaluations is high.

Future studies seem intriguing to enhance the state-of-the-art evolutionary algorithm EAX for the TSP in terms of EDO. 
Besides, the application of high-order entropy diversity optimisation into EAs for other combinatorial optimisation problems seems to be an interesting step.
\section*{acknowledgements}
This work was supported by the Australian Research Council through grant DP190103894 and by the South Australian Government through the Research Consortium "Unlocking Complex Resources through Lean Processing".

\bibliographystyle{abbrvnat}
\bibliography{references}

\end{document}

%% file: pic1.tex
\def \n {15}
\def \radius {2cm}
\def \margin {8} 
\begin{tikzpicture}[scale=1]
\foreach \s in {1,...,\n}
{
  \node[draw, circle] at ({360/\n * (\s - 1)}:\radius) {};
  \draw[-, >=latex] ({360/\n * (\s - 1)+\margin}:\radius) 
    arc ({360/\n * (\s - 1)+\margin}:{360/\n * (\s)-\margin}:\radius);

}
\draw[->, >=latex] ({360/\n * (1 - 1)}:\radius+12) 
    arc ({360/\n * (1 - 1)}:{360/\n * (1+1)}:\radius+12);
\draw[->, >=latex] ({360/\n * (2 - 1)}:\radius+18) 
    arc ({360/\n * (2 - 1)}:{360/\n * (2+1)}:\radius+18);
\draw[->, >=latex] ({360/\n * (3 - 1)}:\radius+24) 
    arc ({360/\n * (3 - 1)}:{360/\n * (3+1)}:\radius+24);
    
\draw[->, >=latex] ({360/\n * (1 - 1)}:\radius+15) 
    arc ({360/\n * (1 - 1)}:{360/\n * (-2)}:\radius+15);
\draw[->, >=latex] ({360/\n * (-1)}:\radius+21) 
    arc ({360/\n * (-1)}:{360/\n * (-3)}:\radius+21);
\draw[->, >=latex] ({360/\n * (-2)}:\radius+27) 
    arc ({360/\n * (-2)}:{360/\n * (-4)}:\radius+27);
\node[rotate=26] at (-1.2,2.6) {$\cdots$};
\node[rotate=-26] at (-1.2,-2.6) {$\cdots$};
\end{tikzpicture}

%% file: pic5.tex
\begin{tikzpicture}
\begin{axis}[
    width=5cm,
    xlabel={Iteration},
    ylabel={Entropy},
    legend columns=2,
    yticklabel style={rotate=90},
    xticklabel style={font=\scriptsize},
    yticklabel style={font=\scriptsize},
    legend style={draw=none, at={(0.5,1.45)}, anchor=south, /tikz/every even column/.append style={column sep=0.5pt}},
         legend image code/.code={
              \draw[#1] (0cm,-0.1cm) -- (0.3cm,0.1cm);
         },
    title = {$\mu$ = 25}
]
\addplot[color=orange, thick] coordinates {
(100,0.9165)
(200,0.9265)
(300,0.9337)
(400,0.941)
(500,0.9471)
(600,0.9528)
(700,0.9583)
(800,0.9635)
(900,0.9679)
(1000,0.972)
(1100,0.9759)
(1200,0.9796)
(1300,0.9829)
(1400,0.9862)
(1500,0.9891)
(1600,0.992)
(1700,0.9943)
(1800,0.9961)
(1900,0.9974)
(2000,0.9987)
(2100,0.9993)
(2200,0.9997)
(2300,0.9999)
(2400,0.9999)
(2500,1)
(2600,1)
(2700,1)
(2800,1)
(2900,1)
(3000,1)
(3100,1)
(3200,1)
(3300,1)
(3400,1)
(3500,1)
(3600,1)
(3700,1)
(3800,1)
(3900,1)
(4000,1)
(4100,1)
(4200,1)
(4300,1)
(4400,1)
(4500,1)
(4600,1)
(4700,1)
(4800,1)
(4900,1)
(5000,1)
(5100,1)
(5200,1)
(5300,1)
(5400,1)
(5500,1)
(5600,1)
(5700,1)
(5800,1)
(5900,1)
(6000,1)
};
\addplot[color=teal, thick] coordinates {
(100,0.9106)
(200,0.9162)
(300,0.9221)
(400,0.9274)
(500,0.932)
(600,0.9367)
(700,0.9406)
(800,0.944)
(900,0.9473)
(1000,0.9503)
(1100,0.9534)
(1200,0.9557)
(1300,0.9582)
(1400,0.9602)
(1500,0.9621)
(1600,0.9638)
(1700,0.9656)
(1800,0.9673)
(1900,0.9692)
(2000,0.9706)
(2100,0.9722)
(2200,0.9737)
(2300,0.9748)
(2400,0.9762)
(2500,0.9773)
(2600,0.9783)
(2700,0.9793)
(2800,0.9803)
(2900,0.9813)
(3000,0.9822)
(3100,0.983)
(3200,0.984)
(3300,0.9847)
(3400,0.9854)
(3500,0.9861)
(3600,0.9867)
(3700,0.9871)
(3800,0.9876)
(3900,0.9882)
(4000,0.9886)
(4100,0.9891)
(4200,0.9896)
(4300,0.9901)
(4400,0.9905)
(4500,0.991)
(4600,0.9913)
(4700,0.9917)
(4800,0.9921)
(4900,0.9925)
(5000,0.9928)
(5100,0.9931)
(5200,0.9934)
(5300,0.9936)
(5400,0.9939)
(5500,0.9942)
(5600,0.9944)
(5700,0.9946)
(5800,0.9948)
(5900,0.995)
(6000,0.9951)
};

\end{axis}
\end{tikzpicture}
\begin{tikzpicture}
\begin{axis}[
    width=5cm,
    xlabel={Iteration},
    yticklabel style={rotate=90},
    xticklabel style={font=\scriptsize},
    yticklabel style={font=\scriptsize},
    legend columns=2,
    legend style={at={(0.5,1.45)}, anchor=south}, 
    title = {$\mu$ = 125}
]
\addplot[color=orange, thick] coordinates {
(1000,0.96239)
(2000,0.97141)
(3000,0.97826)
(4000,0.983)
(5000,0.98666)
(6000,0.98946)
(7000,0.99176)
(8000,0.99356)
(9000,0.99503)
(10000,0.99619)
(11000,0.99714)
(12000,0.99784)
(13000,0.99835)
(14000,0.99871)
(15000,0.99897)
(16000,0.99915)
(17000,0.99933)
(18000,0.99943)
(19000,0.99954)
(20000,0.99964)
};
\addplot[color=teal, thick] coordinates {
(1000,0.95946)
(2000,0.96718)
(3000,0.97344)
(4000,0.97816)
(5000,0.98184)
(6000,0.98483)
(7000,0.98728)
(8000,0.98921)
(9000,0.99075)
(10000,0.99204)
(11000,0.9931)
(12000,0.994)
(13000,0.99475)
(14000,0.99539)
(15000,0.99593)
(16000,0.99642)
(17000,0.99683)
(18000,0.99719)
(19000,0.9975)
(20000,0.99778)
};
\end{axis}
\end{tikzpicture}
\begin{tikzpicture}
\begin{axis}[
    width=5cm,
    xlabel={Iteration},
    legend columns=2,
    yticklabel style={rotate=90},
    xticklabel style={font=\scriptsize},
    yticklabel style={font=\scriptsize},
    legend style={at={(0.5,1.45)}, anchor=south},
    title = {$\mu$ = 250}
]
\addplot[color=orange, thick] coordinates {
(1000,0.97908)
(2000,0.98283)
(3000,0.9858)
(4000,0.98823)
(5000,0.99021)
(6000,0.99182)
(7000,0.99313)
(8000,0.99423)
(9000,0.99516)
(10000,0.99587)
(11000,0.99641)
(12000,0.99685)
(13000,0.99718)
(14000,0.99746)
(15000,0.99769)
(16000,0.9979)
(17000,0.99808)
(18000,0.99823)
(19000,0.99836)
(20000,0.99846)
(21000,0.99856)
(22000,0.99864)
(23000,0.99874)
(24000,0.9988)
(25000,0.99885)
(26000,0.9989)
(27000,0.99897)
(28000,0.99903)
(29000,0.99905)
(30000,0.9991)
};
\addplot[color=teal, thick] coordinates {
(1000,0.97749)
(2000,0.98044)
(3000,0.98308)
(4000,0.98531)
(5000,0.98716)
(6000,0.98875)
(7000,0.99016)
(8000,0.99131)
(9000,0.99231)
(10000,0.99316)
(11000,0.9939)
(12000,0.99451)
(13000,0.99505)
(14000,0.99549)
(15000,0.99587)
(16000,0.99623)
(17000,0.99651)
(18000,0.99677)
(19000,0.99703)
(20000,0.99723)
(21000,0.99739)
(22000,0.99754)
(23000,0.99769)
(24000,0.99782)
(25000,0.99792)
(26000,0.998)
(27000,0.9981)
(28000,0.99818)
(29000,0.99826)
(30000,0.99831)
};
\end{axis}
\end{tikzpicture}

%% file: pic4.tex
\begin{tikzpicture}[scale = 0.6]
\begin{groupplot}[
    group style ={group size= 3 by 1, ylabels at=edge left, xlabels at=edge bottom}, yticklabel style={text width=0.035\textwidth, align=right, inner xsep=0pt, xshift=-0.005\textwidth},ylabel=High-order Entropy, ylabel style={text height=0.02\textwidth,inner ysep=0pt}, 
    legend columns=-1,
        legend style={draw=none, /tikz/every even column/.append style={column sep=0.5pt}},
         legend image code/.code={
              \draw[#1] (0cm,-0.1cm) -- (0.15cm,0.1cm);
         }
        ]
    \nextgroupplot[title = K is equal to 2,symbolic x coords={0.1,0.2,0.3,0.4,0.5,0.6,0.7,0.8,0.9,1,1.1,1.2,1.3,1.4,1.5,1.6,1.7,1.8,1.9,2},xtick={0.1,0.5,1,1.5,2}, legend to name=grouplegend,]
    \addplot[orange  ,dotted, thick] 
    plot coordinates {
            (0.1,5.6415)
            (0.2,5.6755)
            (0.3,5.7001)
            (0.4,5.7210)
            (0.5,5.7389)
            (0.6,5.7535)
            (0.7,5.7648)
            (0.8,5.7776)
            (0.9,5.7867)
            (1,5.7974)
            (1.1,5.8042)
            (1.2,5.8108)
            (1.3,5.8164)
            (1.4,5.8224)
            (1.5,5.8272)
            (1.6,5.8312)
            (1.7,5.8353)
            (1.8,5.8391)
            (1.9,5.8431)
            (2,5.8464)};
        \addlegendentry{ENT(5\%)};
    \addplot[teal , dotted, thick]
        plot coordinates {
            (0.1,5.6734)
            (0.2,5.7002)
            (0.3,5.7184)
            (0.4,5.7300)
            (0.5,5.7348)
            (0.6,5.7483)
            (0.7,5.7564)
            (0.8,5.7618)
            (0.9,5.7670)
            (1,5.7721)
            (1.1,5.7715)
            (1.2,5.7758)
            (1.3,5.7774)
            (1.4,5.7788)
            (1.5,5.7777)
            (1.6,5.7783)
            (1.7,5.7784)
            (1.8,5.7780)
            (1.9,5.7806)
            (2,5.7830)
        };
        \addlegendentry{ED(5\%)};
    \addplot[violet, dotted, thick]
        plot coordinates {
            (0.1,5.7004)
            (0.2,5.7234)
            (0.3,5.7378)
            (0.4,5.7490)
            (0.5,5.7593)
            (0.6,5.7669)
            (0.7,5.7731)
            (0.8,5.7780)
            (0.9,5.7841)
            (1,5.7892)
            (1.1,5.7928)
            (1.2,5.7969)
            (1.3,5.7998)
            (1.4,5.8032)
            (1.5,5.8065)
            (1.6,5.8104)
            (1.7,5.8123)
            (1.8,5.8142)
            (1.9,5.8164)
            (2,5.8189)
        };
        \addlegendentry{PD(5\%)};
    \addplot[orange, dashed, thick]
    plot coordinates {
            (0.1,5.7875)
            (0.2,5.8394)
            (0.3,5.8822)
            (0.4,5.9141)
            (0.5,5.9437)
            (0.6,5.9667)
            (0.7,5.9865)
            (0.8,6.0035)
            (0.9,6.0172)
            (1,6.0301)
            (1.1,6.0408)
            (1.2,6.0504)
            (1.3,6.0593)
            (1.4,6.0673)
            (1.5,6.0752)
            (1.6,6.0818)
            (1.7,6.0884)
            (1.8,6.0934)
            (1.9,6.0980)
            (2,6.1016)
        };
        \addlegendentry{ENT(10\%)};
    \addplot[teal, dashed, thick]
        plot coordinates {
            (0.1,5.8944)
            (0.2,5.9274)
            (0.3,5.9490)
            (0.4,5.9633)
            (0.5,5.9766)
            (0.6,5.9888)
            (0.7,5.9984)
            (0.8,6.0029)
            (0.9,6.0070)
            (1,6.0086)
            (1.1,6.0114)
            (1.2,6.0128)
            (1.3,6.0151)
            (1.4,6.0227)
            (1.5,6.0252)
            (1.6,6.0241)
            (1.7,6.0257)
            (1.8,6.0261)
            (1.9,6.0263)
            (2,6.0263)
        };
        \addlegendentry{ED(10\%)};
    \addplot[violet, dashed, thick]
        plot coordinates {
            (0.1,5.8664)
            (0.2,5.9052)
            (0.3,5.9345)
            (0.4,5.9563)
            (0.5,5.9715)
            (0.6,5.9860)
            (0.7,5.9970)
            (0.8,6.0065)
            (0.9,6.0164)
            (1,6.0238)
            (1.1,6.0300)
            (1.2,6.0353)
            (1.3,6.0399)
            (1.4,6.0444)
            (1.5,6.0487)
            (1.6,6.0532)
            (1.7,6.0570)
            (1.8,6.0611)
            (1.9,6.0637)
            (2,6.0666)
        };
        \addlegendentry{PD(10\%)};
    \addplot[orange, solid, thick] plot coordinates {
            (0.1,5.9560)
            (0.2,6.0655)
            (0.3,6.1433)
            (0.4,6.2056)
            (0.5,6.2548)
            (0.6,6.2972)
            (0.7,6.3307)
            (0.8,6.3556)
            (0.9,6.3762)
            (1,6.3956)
            (1.1,6.4121)
            (1.2,6.4260)
            (1.3,6.4378)
            (1.4,6.4500)
            (1.5,6.4592)
            (1.6,6.4671)
            (1.7,6.4751)
            (1.8,6.4817)
            (1.9,6.4886)
            (2,6.4942)
        };
        \addlegendentry{ENT(20\%)};
    \addplot[teal, solid, thick]
        plot coordinates {
            (0.1,6.1842)
            (0.2,6.2494)
            (0.3,6.3002)
            (0.4,6.3272)
            (0.5,6.3394)
            (0.6,6.3509)
            (0.7,6.3629)
            (0.8,6.3687)
            (0.9,6.3720)
            (1,6.3742)
            (1.1,6.3772)
            (1.2,6.3775)
            (1.3,6.3735)
            (1.4,6.3754)
            (1.5,6.3768)
            (1.6,6.3784)
            (1.7,6.3796)
            (1.8,6.3832)
            (1.9,6.3792)
            (2,6.3751)
        };
        \addlegendentry{ED(20\%)};
    \addplot[violet, solid, thick]
        plot coordinates {
            (0.1,6.0839)
            (0.2,6.1634)
            (0.3,6.2244)
            (0.4,6.2652)
            (0.5,6.2963)
            (0.6,6.3192)
            (0.7,6.3416)
            (0.8,6.3588)
            (0.9,6.3722)
            (1,6.3831)
            (1.1,6.3936)
            (1.2,6.4039)
            (1.3,6.4109)
            (1.4,6.4173)
            (1.5,6.4231)
            (1.6,6.4283)
            (1.7,6.4326)
            (1.8,6.4377)
            (1.9,6.4415)
            (2,6.4449)
        };
        \addlegendentry{PD(20\%)};
    \coordinate (top) at (rel axis cs:0,1);
    \nextgroupplot[title = K is equal to 3,symbolic x coords={0.1,0.2,0.3,0.4,0.5,0.6,0.7,0.8,0.9,1,1.1,1.2,1.3,1.4,1.5,1.6,1.7,1.8,1.9,2},xtick={0.1,0.5,1,1.5,2}, xlabel = {Number of fitness evaluation}]
    \addplot[orange, dotted, thick]   plot coordinates {
            (0.1,6.0121)
            (0.2,6.0803)
            (0.3,6.1327)
            (0.4,6.1700)
            (0.5,6.2006)
            (0.6,6.2309)
            (0.7,6.2562)
            (0.8,6.2813)
            (0.9,6.3026)
            (1,6.3196)
            (1.1,6.3388)
            (1.2,6.3557)
            (1.3,6.3715)
            (1.4,6.3822)
            (1.5,6.3934)
            (1.6,6.4067)
            (1.7,6.4168)
            (1.8,6.4261)
            (1.9,6.4322)
            (2,6.4381)
        };
    \addplot[teal, dotted, thick]   plot coordinates {
            (0.1,6.0204)
            (0.2,6.0838)
            (0.3,6.1231)
            (0.4,6.1484)
            (0.5,6.1575)
            (0.6,6.1826)
            (0.7,6.1981)
            (0.8,6.2065)
            (0.9,6.2170)
            (1,6.2285)
            (1.1,6.2248)
            (1.2,6.2316)
            (1.3,6.2340)
            (1.4,6.2387)
            (1.5,6.2360)
            (1.6,6.2351)
            (1.7,6.2300)
            (1.8,6.2298)
            (1.9,6.2353)
            (2,6.2393)
        };
    \addplot[violet, dotted, thick]   plot coordinates {
            (0.1,6.0918)
            (0.2,6.1463)
            (0.3,6.1804)
            (0.4,6.2074)
            (0.5,6.2304)
            (0.6,6.2482)
            (0.7,6.2623)
            (0.8,6.2743)
            (0.9,6.2877)
            (1,6.2996)
            (1.1,6.3080)
            (1.2,6.3176)
            (1.3,6.3240)
            (1.4,6.3318)
            (1.5,6.3387)
            (1.6,6.3462)
            (1.7,6.3507)
            (1.8,6.3556)
            (1.9,6.3603)
            (2,6.3659)
        };
    \addplot[orange, dashed, thick]   plot coordinates {
            (0.1,6.2648)
            (0.2,6.3834)
            (0.3,6.4731)
            (0.4,6.5503)
            (0.5,6.6156)
            (0.6,6.6674)
            (0.7,6.7150)
            (0.8,6.7572)
            (0.9,6.7946)
            (1,6.8268)
            (1.1,6.8506)
            (1.2,6.8759)
            (1.3,6.9006)
            (1.4,6.9247)
            (1.5,6.9426)
            (1.6,6.9585)
            (1.7,6.9754)
            (1.8,6.9927)
            (1.9,7.0045)
            (2,7.0154)
        };
    \addplot[teal, dashed, thick]   plot coordinates {
            (0.1,6.4341)
            (0.2,6.5204)
            (0.3,6.5727)
            (0.4,6.6130)
            (0.5,6.6472)
            (0.6,6.6710)
            (0.7,6.6985)
            (0.8,6.7114)
            (0.9,6.7222)
            (1,6.7262)
            (1.1,6.7532)
            (1.2,6.7572)
            (1.3,6.7546)
            (1.4,6.7594)
            (1.5,6.7618)
            (1.6,6.7610)
            (1.7,6.7594)
            (1.8,6.7618)
            (1.9,6.7610)
            (2,6.7596)
        };
    \addplot[violet, dashed, thick]   plot coordinates {
            (0.1,6.3557)
            (0.2,6.4556)
            (0.3,6.5270)
            (0.4,6.5795)
            (0.5,6.6168)
            (0.6,6.6514)
            (0.7,6.6792)
            (0.8,6.7020)
            (0.9,6.7268)
            (1,6.7463)
            (1.1,6.7621)
            (1.2,6.7758)
            (1.3,6.7874)
            (1.4,6.7984)
            (1.5,6.8102)
            (1.6,6.8201)
            (1.7,6.8293)
            (1.8,6.8385)
            (1.9,6.8448)
            (2,6.8513)
        };
    \addplot[orange, solid, thick]   plot coordinates {
            (0.1,6.6208)
            (0.2,6.8542)
            (0.3,7.0247)
            (0.4,7.1481)
            (0.5,7.2678)
            (0.6,7.3674)
            (0.7,7.4451)
            (0.8,7.5122)
            (0.9,7.5672)
            (1,7.6153)
            (1.1,7.6566)
            (1.2,7.6936)
            (1.3,7.7259)
            (1.4,7.7552)
            (1.5,7.7870)
            (1.6,7.8093)
            (1.7,7.8325)
            (1.8,7.8501)
            (1.9,7.8673)
            (2,7.8819)
        };
    \addplot[teal, solid, thick]   plot coordinates {
            (0.1,6.8936)
            (0.2,7.0505)
            (0.3,7.1729)
            (0.4,7.2463)
            (0.5,7.2875)
            (0.6,7.3258)
            (0.7,7.3645)
            (0.8,7.3789)
            (0.9,7.3920)
            (1,7.3992)
            (1.1,7.4120)
            (1.2,7.4149)
            (1.3,7.4111)
            (1.4,7.4193)
            (1.5,7.4251)
            (1.6,7.4324)
            (1.7,7.4388)
            (1.8,7.4473)
            (1.9,7.4383)
            (2,7.4296)
        };
    \addplot[violet, solid, thick]   plot coordinates {
            (0.1,6.7041)
            (0.2,6.8931)
            (0.3,7.0369)
            (0.4,7.1338)
            (0.5,7.2099)
            (0.6,7.2704)
            (0.7,7.3288)
            (0.8,7.3750)
            (0.9,7.4080)
            (1,7.4381)
            (1.1,7.4654)
            (1.2,7.4932)
            (1.3,7.5117)
            (1.4,7.5307)
            (1.5,7.5473)
            (1.6,7.5627)
            (1.7,7.5763)
            (1.8,7.5917)
            (1.9,7.6028)
            (2,7.6136)
        };
    \nextgroupplot[title = K is equal to 4,symbolic x coords={0.1,0.2,0.3,0.4,0.5,0.6,0.7,0.8,0.9,1,1.1,1.2,1.3,1.4,1.5,1.6,1.7,1.8,1.9,2},xtick={0.1,0.5,1,1.5,2}]
       \addplot[orange, dotted, thick]   plot coordinates {
            (0.1,6.3060)
            (0.2,6.3918)
            (0.3,6.4593)
            (0.4,6.5154)
            (0.5,6.5607)
            (0.6,6.5980)
            (0.7,6.6350)
            (0.8,6.6660)
            (0.9,6.6928)
            (1,6.7135)
            (1.1,6.7379)
            (1.2,6.7588)
            (1.3,6.7764)
            (1.4,6.7937)
            (1.5,6.8112)
            (1.6,6.8223)
            (1.7,6.8358)
            (1.8,6.8482)
            (1.9,6.8575)
            (2,6.8654)
        };
    \addplot[teal, dotted, thick]   plot coordinates {
            (0.1,6.2410)
            (0.2,6.3216)
            (0.3,6.3687)
            (0.4,6.3977)
            (0.5,6.4069)
            (0.6,6.4365)
            (0.7,6.4539)
            (0.8,6.4619)
            (0.9,6.4731)
            (1,6.4854)
            (1.1,6.4798)
            (1.2,6.4863)
            (1.3,6.4881)
            (1.4,6.4919)
            (1.5,6.4856)
            (1.6,6.4831)
            (1.7,6.4743)
            (1.8,6.4746)
            (1.9,6.4788)
            (2,6.4837)
        };
    \addplot[violet, dotted, thick]   plot coordinates {
            (0.1,6.3453)
            (0.2,6.4189)
            (0.3,6.4661)
            (0.4,6.5027)
            (0.5,6.5343)
            (0.6,6.5577)
            (0.7,6.5764)
            (0.8,6.5935)
            (0.9,6.6116)
            (1,6.6268)
            (1.1,6.6380)
            (1.2,6.6510)
            (1.3,6.6593)
            (1.4,6.6700)
            (1.5,6.6787)
            (1.6,6.6892)
            (1.7,6.6954)
            (1.8,6.7017)
            (1.9,6.7088)
            (2,6.7158)
        };
    \addplot[orange, dashed, thick]   plot coordinates {
            (0.1,6.6488)
            (0.2,6.8205)
            (0.3,6.9509)
            (0.4,7.0589)
            (0.5,7.1417)
            (0.6,7.2098)
            (0.7,7.2826)
            (0.8,7.3409)
            (0.9,7.3935)
            (1,7.4368)
            (1.1,7.4711)
            (1.2,7.4965)
            (1.3,7.5301)
            (1.4,7.5602)
            (1.5,7.5845)
            (1.6,7.6086)
            (1.7,7.6317)
            (1.8,7.6511)
            (1.9,7.6734)
            (2,7.6890)
        };
    \addplot[teal, dashed, thick]   plot coordinates {
            (0.1,6.6775)
            (0.2,6.8901)
            (0.3,6.9544)
            (0.4,7.0043)
            (0.5,7.0469)
            (0.6,7.0776)
            (0.7,7.1123)
            (0.8,7.1272)
            (0.9,7.1412)
            (1,7.1455)
            (1.1,7.1498)
            (1.2,7.1525)
            (1.3,7.1506)
            (1.4,7.1725)
            (1.5,7.1775)
            (1.6,7.1711)
            (1.7,7.1756)
            (1.8,7.1763)
            (1.9,7.1743)
            (2,7.1673)
        };
    \addplot[violet, dashed, thick]   plot coordinates {
            (0.1,6.6768)
            (0.2,6.8076)
            (0.3,6.9020)
            (0.4,6.9706)
            (0.5,7.0192)
            (0.6,7.0653)
            (0.7,7.1019)
            (0.8,7.1333)
            (0.9,7.1669)
            (1,7.1944)
            (1.1,7.2163)
            (1.2,7.2163)
            (1.3,7.2516)
            (1.4,7.2679)
            (1.5,7.2848)
            (1.6,7.2988)
            (1.7,7.3115)
            (1.8,7.3237)
            (1.9,7.3332)
            (2,7.3425)
        };
    \addplot[orange, solid, thick]   plot coordinates {
            (0.1,7.1917)
            (0.2,7.5219)
            (0.3,7.7637)
            (0.4,7.9427)
            (0.5,8.0876)
            (0.6,8.2102)
            (0.7,8.3075)
            (0.8,8.3930)
            (0.9,8.4693)
            (1,8.5285)
            (1.1,8.5809)
            (1.2,8.6300)
            (1.3,8.6738)
            (1.4,8.7158)
            (1.5,8.7486)
            (1.6,8.7790)
            (1.7,8.8071)
            (1.8,8.8351)
            (1.9,8.8592)
            (2,8.8808)
        };
    \addplot[teal, solid, thick]   plot coordinates {
            (0.1,7.3296)
            (0.2,7.5216)
            (0.3,7.6732)
            (0.4,7.7604)
            (0.5,7.8173)
            (0.6,7.8714)
            (0.7,7.9232)
            (0.8,7.9385)
            (0.9,7.9547)
            (1,7.9582)
            (1.1,7.9750)
            (1.2,7.9793)
            (1.3,7.9715)
            (1.4,7.9839)
            (1.5,7.9921)
            (1.6,8.0022)
            (1.7,8.0178)
            (1.8,8.0294)
            (1.9,8.0159)
            (2,8.0051)
        };
    \addplot[violet, solid, thick]   plot coordinates {
            (0.1,7.1073)
            (0.2,7.3468)
            (0.3,7.5311)
            (0.4,7.6536)
            (0.5,7.7538)
            (0.6,7.8349)
            (0.7,7.9138)
            (0.8,7.9778)
            (0.9,8.0240)
            (1,8.0650)
            (1.1,8.1029)
            (1.2,8.1415)
            (1.3,8.1673)
            (1.4,8.1918)
            (1.5,8.2159)
            (1.6,8.2370)
            (1.7,8.2568)
            (1.8,8.2787)
            (1.9,8.2941)
            (2,8.3107)
        };
\end{groupplot}
\begin{tiny}
\node  at (11.5,-1.5) {\ref{grouplegend}};
\end{tiny}
\end{tikzpicture}